\documentclass{article} 
\usepackage{iclr2027_conference,times}

\usepackage{amsmath,amsfonts,bm}

\def\eqref#1{equation~\ref{#1}}

\def\1{\bm{1}}

\DeclareMathAlphabet{\mathsfit}{\encodingdefault}{\sfdefault}{m}{sl}
\SetMathAlphabet{\mathsfit}{bold}{\encodingdefault}{\sfdefault}{bx}{n}

\usepackage{wrapfig}
\usepackage{hyperref}
\usepackage{url}
\usepackage[utf8]{inputenc} 
\usepackage[T1]{fontenc}    
\usepackage{hyperref}       
\usepackage{url}            
\usepackage{booktabs}       
\usepackage{amsfonts}       
\usepackage{nicefrac}       
\usepackage{microtype}      
\usepackage{xcolor}         
\usepackage{xpatch}
\usepackage{multirow}
\usepackage{multicol}
\usepackage{graphicx}  %
\usepackage{subcaption}
\usepackage{amsmath} 
\usepackage{amssymb}
\usepackage{wrapfig}
\usepackage[table]{xcolor} 

\usepackage{subcaption}  

\newtheorem{proposition}{Proposition}
\newtheorem{Theorem}{Theorem}
\newtheorem{proof}{proof}
\newtheorem{Remark}{Remark}

\title{TWIN-MoE: Two-stage Intra-Expert Imputation Network for Incomplete Multimodal Learning}

\iclrfinalcopy
\author{Liang Nathan Zheng, Jiyuan Liu, Wenhao Liang, Wei Emma Zhang, Weitong Chen\thanks{Corresponding Author. The paper is under review}, \& Lin Yue  \\
School of Computer Science and Information Technology \\
Adelaide University\\
Adelaide, SA 5000, Australia \\
\texttt{\{liangwei.zheng,jiyuan.liu,wenhao.liang,wei.e.zhang,weitong.chen,lin.yue\}} \\
\texttt{@adelaide.edu.au} \\
}

\begin{document}

\maketitle

\begin{abstract}
Multimodal learning is often deployed under the assumption that all modalities are fully observed, yet real-world data frequently violates this assumption, leaving some modalities missing. Recent Mixture-of-Experts (MoE) approaches address missing modalities either by routing each missingness pattern to a dedicated expert, which scales poorly as the number of modalities grows, or by imputing the missing modality from representations contributed by multiple experts, which requires a costly retrieval mechanism. We argue that such unconstrained inter-expert imputation can be sub-optimal, as it aggregates features from distinct routing-induced semantic regions and may introduce information unrelated to the target modality. Building on a geometric interpretation of sparse MoE routing through expert-specific Voronoi regions, we show that routing induces expert-conditioned representation neighborhoods and derive an imputation-error bound controlled by their local feature dispersion. When these expert-conditioned neighborhoods are more compact than the global representation space, intra-expert aggregation admits a tighter worst-case reconstruction bound than unconstrained global aggregation. Guided by this analysis, we propose TWIN, a two-stage intra-expert imputation framework. Pre-imputation generates a context-conditioned estimate in the encoder embedding space and radially soft-clips it using modality-specific statistics maintained in an observed FIFO queue, this ensures that the pre-imputed modality stays in empirically calibrated region of the target modality space. Post-imputation refines the missing representation after MoE by aggregating expert-conditioned observed representations within the same sample. Across four missing-modality benchmarks under natural, random, and asymmetric dropout protocols, TWIN achieves state-of-the-art performance while using substantially fewer parameters than competing MoE-based imputation methods.
\end{abstract}

\section{Introduction}
Multimodal learning integrates complementary information from multiple data sources and has demonstrated strong performance across a wide range of downstream tasks, including computer vision \cite{wu2017visual,yan2025automotive,pmlrzhu24a,wang2024improving}, temporal reasoning \cite{zheng2024irregularity,zheng2026lifting,wu2026vision,ge2026vasevqa}, and sentiment analysis \cite{han2024fusemoe,zheng2025rethinking}. However, most existing multimodal learning methods assume that all modalities are fully observed \cite{rahman2020integrating,dai2026papervoyager,zhang2025advgan}, an assumption that is frequently violated in real-world applications due to heterogeneous data availability. For example, a patient may have clinical notes but lack corresponding laboratory measurements because not all examinations are performed for every patient. Such missing modalities result in incomplete signals and can substantially degrade downstream prediction performance, motivating robust multimodal models that can operate under arbitrary modality availability. Recent studies have explored Mixture-of-Experts (MoE) architectures \cite{han2024fusemoe,yun2024flex,zheng2026tackling} to address this problem by routing different missingness patterns to specialized experts. However, the number of possible modality combinations grows rapidly with the number of modalities, leading to poor scalability. Another line of MoE-based methods \cite{zheng2025rethinking} adopts two-stage imputation, where multiple experts provide complementary representations to reconstruct a missing modality. Such approaches, however, require additional retrieval mechanisms to determine which expert representations should contribute to the reconstruction, introducing considerable parameter and computational overhead. More importantly, inter-expert imputation may introduce additional geometric uncertainty because reference representations can originate from distinct routing regions. To motivate this observation, we visualize the routing behavior of an MoE layer in Figure \ref{fig:t-sne_plot_sub_spaces}, where tokens assigned to different experts form localized regions in the embedding space.

\begin{wrapfigure}{r}{0.32\textwidth}
    \centering
    \includegraphics[width=\linewidth]{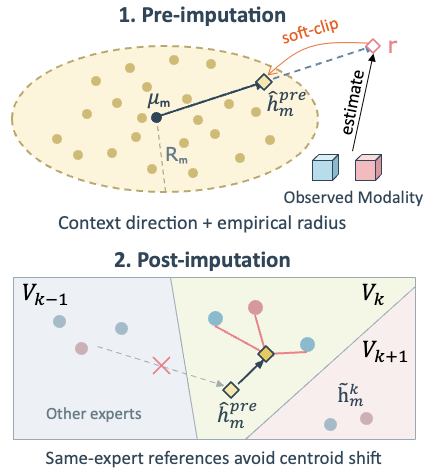}
    \caption{Illustration of the intuition behind TWIN-MoE.}
    \label{fig:TWINMoE_intuition}
\end{wrapfigure}

Motivated by this routing geometry, we analyze missing-modality imputation from the perspective of Voronoi tessellation. We show that sparse MoE routing induces expert-specific regions in the representation space and, under locally compact expert regions, restricting imputation references to the same routing region admits a tighter worst-case reconstruction-error bound than unconstrained inter-expert aggregation. We further characterize the geometric uncertainty introduced by cross-expert references through the additional centroid-shift term in Eq. \ref{eq:shift_error}. Guided by this analysis, we propose \textbf{TWIN}, a two-stage intra-expert imputation framework consisting of distribution-aware pre-imputation and locality-aware post-imputation. As illustrated in Figure \ref{fig:TWINMoE_intuition}, the pre-imputation stage predicts a context-dependent representation for the missing modality from the observed modalities and applies radial soft-clipping using modality-specific empirical statistics to prevent the estimate from drifting excessively from the observed distribution. After MoE routing, the post-imputation stage further refines the missing representation using only observed tokens routed to the same expert, thereby avoiding unconstrained information transfer across geometrically distinct routing regions. Our theoretical analysis focuses on sparse MoE routing, while the same locality principle can also be instantiated with Soft MoE \cite{puigcerver2024sparse}, which we evaluate empirically.

Compared with existing two-stage imputation methods that initialize missing representations through sampled or globally aggregated features, our pre-imputation explicitly uses the observed modalities to determine a context-dependent imputation direction while constraining its magnitude according to the empirical distribution of the target modality. Meanwhile, intra-expert post-imputation restricts reference representations to the local routing region of the missing token, providing a simple alternative to parameter-heavy inter-expert retrieval mechanisms. Our main contributions are summarized as follows:

\begin{itemize}
\item We propose TWIN, a two-stage pre- and post-imputation framework for incomplete multimodal learning that can be integrated with both sparse and soft MoE backbones without substantial architectural modification.
\item We provide a geometric analysis of sparse MoE routing and show that, when expert-specific routing regions are more compact than the global representation space, intra-expert aggregation admits a tighter worst-case reconstruction-error upper bound than unconstrained inter-expert aggregation.
\item We evaluate TWIN on four missing-modality benchmarks under natural missingness, random modality dropout, and asymmetric modality dropout, where it consistently achieves competitive or state-of-the-art performance against existing missing-modality and MoE-based methods.
\end{itemize}

\section{Related Work}
\subsection{Multimodal Fusion and Incomplete Multimodal Learning}
Multimodal learning integrates heterogeneous information from multiple sources to improve downstream prediction. Representative approaches include kernel-based fusion~\cite{bucak2013multiple}, early/late fusion~\cite{xu2023multimodal,baltruvsaitis2018multimodal,guo2019deep}, Tensor Fusion Network~\cite{zadeh2017tensor}, Multimodal Transformer~\cite{tsai2019multimodal}, Multimodal Adaptation Gate~\cite{rahman2020integrating}, and LlMoE~\cite{mustafa2022multimodal}. Similar strategies have also been applied to healthcare, where methods such as MedFuse~\cite{hayat2022medfuse}, MISTS~\cite{zhang2023improving}, and attention-gated fusion~\cite{yang2021leverage} combine clinical modalities including EHR, irregular time series, and chest X-rays. However, many multimodal methods rely on complete modality observations. To handle incomplete inputs, SMIL~\cite{ma2021smil} reconstructs latent representations through Bayesian meta-learning, while generative approaches~\cite{du2018semi,shang2017vigan} synthesize missing views. ShaSpec~\cite{wang2023multi} and DrFuse~\cite{yao2024drfuse} further separate shared and modality-specific representations to improve robustness. Although effective, these methods often introduce additional reconstruction modules or auxiliary objectives, increasing model and optimization complexity~\cite{sun2024similar}.

\subsection{Mixture-of-Experts for Missing-Modality Learning}
Mixture-of-Experts (MoE) provides a flexible way to handle incomplete multimodal inputs through expert specialization and dynamic routing. FuseMoE~\cite{han2024fusemoe} models different modality combinations with specialized experts, while Flex-MoE~\cite{yun2024flex} introduces a modality bank to support arbitrary modality availability. ConfSMoE~\cite{zheng2025rethinking} further combines confidence-aware routing with two-stage imputation using intra- and inter-modal information. More recently, MCMoE~\cite{xu2026mcmoe} employs an adaptive modality generator together with modality-specific experts to reconstruct and refine missing representations, while MARS~\cite{baek2026mars} specializes experts according to residual discrepancies between complete and incomplete representations. CP-IMoE~\cite{li2026cpimoe} explores a collaborative prompt-guided interactive MoE framework for incomplete multimodal learning.  Despite these advances, existing approaches mainly use MoE for missing-pattern specialization, modality completion, or adaptive expert interaction. In contrast, TWIN exploits the locality already induced by MoE routing and performs distribution-aware pre-imputation followed by intra-expert refinement, avoiding unrestricted cross-expert aggregation while introducing only minor modifications to the underlying MoE backbone.

\section{From Mixture-of-Experts to Voronoi Tessellation}
Specifically, given an input data sequence $\mathbf{x} \in \mathbb{R}^{s \times d}$, a Sparse Mixture-of-Experts (Sparse MoE) layer with $N$ experts employs a gating network (often referred to as a router) to dynamically assign tokens to a specialized subset of experts. Unlike expert-choice soft routing, Sparse MoE operates under a \textit{token-choice} paradigm. Formally, the router is parameterized by a weight matrix $\mathbf{W} \in \mathbb{R}^{d \times N}$, which computes the routing probability of the $i$-th token being assigned to the $j$-th expert:

\begin{equation}
    P_{i, j} = \frac{\exp(\mathbf{x}_i \cdot \mathbf{W}_{j} + b_j)}{\sum_{j'=1}^{N} \exp(\mathbf{x}_i \cdot \mathbf{W}_{j'} + b_{j'})}
    \label{eq:routing_prob}
\end{equation}

where $i \in \{1, \dots, s\}$ is the token index, $j \in \{1, \dots, N\}$ is the expert index, and $\mathbf{W}_{j}$ denotes the $j$-th column of the router weights. To induce sparsity and restrict the computational cost, Sparse MoE employs a Top-$K$ hard routing mechanism. Instead of fully activating all experts, each token is strictly dispatched to its highest-scoring $K$ experts (typically $K \ll N$). This discrete assignment can be formalized by a binary routing mask $\mathbf{M} \in \{0, 1\}^{s \times N}$:

\begin{equation}
    \mathbf{z}_i = \sum_{j=1}^N M_{i, j} \cdot P_{i, j} \cdot f_j(\mathbf{x}_i), \quad    
    M_{i, j} = 
    \begin{cases} 
        1, & \text{if } j \in \text{TopK}(P_i) \\ 
        0, & \text{otherwise} 
    \end{cases}
    \label{eq:sparse_combine}
\end{equation}

where $\text{TopK}(P_i)$ retrieves the indices of the $K$ largest probabilities in the distribution $P_i$ and $f_j(\cdot)$ represents the Feed-Forward Network (FFN) block associated with the $j$-th expert. Through this hard masking, the continuous feature space is geometrically partitioned into discrete sub-regions (e.g., Voronoi cells), strictly bounding the receptive field of each expert. Next, each expert $f_j$ processes its assigned token inputs independently. Finally, to reconstruct the output sequence, the token-level output $\mathbf{z}_i \in \mathbb{R}^d$ is formulated as a probabilistically weighted sum of the representations processed by its explicitly selected experts. Next, we build the connection of the Voronoi cell with MoE set-up.

Let $\mathcal{H} \subseteq \mathbb{R}^d$ denote the continuous feature space of the observed modality. In Sparse MoE, the routing mechanism is governed by a gating weight matrix $\mathbf{W} \in \mathbb{R}^{d \times N}$. We conceptualize the column vectors of this matrix as a set of $N$ semantic centroids, denoted as $\mathcal{C} = \{c_1, c_2, \dots, c_N\} \subset \mathcal{H}$, where $c_k = \mathbf{W}_{k}$.

\begin{proposition}[Geometric Partitioning of Sparse Routing via Voronoi Tessellation]
\label{prop:voronoi_partition}

Let $\mathcal{H}\subset\mathbb{R}^{d}$ denote the bounded continuous feature space, and let $\mathcal{C}=\{c_1,c_2,\ldots,c_N\}$ denote the routing centroids associated with $N$ experts. The hard routing function $f_{\mathcal{V}}:\mathcal{H}\rightarrow\{1,\ldots,N\}$ routes a token $h\in\mathcal{H}$ to expert $k$ if it maximizes the gating logit, i.e., $h\cdot c_k+b_k\geq h\cdot c_j+b_j$, $\forall j\neq k$. As formally derived in Appendix~\ref{app:voronoi_proof}, by imposing the routing bias $b_k=-\frac{1}{2}\|c_k\|_2^2$, the linear logit maximization becomes strictly equivalent to nearest-centroid assignment under the squared $L_2$ distance. In experiment, we dynamically calculate the $b_k$ in router to make Voronoi router. Consequently, the continuous feature space is geometrically partitioned into expert-specific regions, where the receptive field of the $k$-th expert corresponds to the $k$-th Voronoi cell:
\begin{equation}
    \mathcal{V}_k
    =
    \left\{
    h\in\mathcal{H}
    \mid
    \|h-c_k\|_2
    \leq
    \|h-c_j\|_2,
    \forall j\neq k
    \right\}.
\end{equation}

\end{proposition}

Furthermore, let $D_{\mathrm{global}}=\sup_{h_a,h_b\in\mathcal{H}}\|h_a-h_b\|_2$ denote the maximum geometric diameter of the global feature space. Correspondingly, the local diameter of the $k$-th Voronoi cell is defined as $D_k=\sup_{h_a,h_b\in\mathcal{V}_k}\|h_a-h_b\|_2$. Since $\mathcal{V}_k\subseteq\mathcal{H}$, the spatial variation of representations assigned to an individual expert is geometrically confined within its local cell, yielding $D_k\leq D_{\mathrm{global}}$. When the learned expert-specific region is strictly more compact than the global feature space, i.e., $D_k<D_{\mathrm{global}}$, this localized geometry provides the basis for the tighter reconstruction-error bound established in Theorem~\ref{thm:local_imputation}.

\section{Proposed Method}

\subsection{Theoretical Contribution}
\begin{Theorem}[Bounded Error of Intra-Expert Imputation]
\label{thm:local_imputation}

For a target representation and its reference representations routed to the
same Voronoi cell $V_k$, the reconstruction error of intra-expert imputation
is upper-bounded by
\begin{equation}
\boxed{
E_{\mathrm{local}}
=
\|\hat{h}_m-h_t^{*}\|_2
\leq
\epsilon_{\mathrm{align}}
+
L D_k ,
}
\label{eq:local_error_bound}
\end{equation}
where $D_k$ denotes the diameter of the corresponding Voronoi cell. In contrast, for unconstrained global imputation over the entire representation space, the worst-case reconstruction error satisfies
\begin{equation}
\boxed{
E_{\mathrm{global}}
\leq
\epsilon_{\mathrm{align}}
+
L_{\rm lip} D_{\mathrm{global}} .
}
\label{eq:global_error_bound}
\end{equation}

\end{Theorem}

Where \(h^*_{m,t}\) denotes the ideal representation of missing modality $m$ for target $t$, $\hat h_{m,t}$ is its imputed estimate, $\epsilon_{\rm align}$ bounds cross-modal misalignment, and $L_{\rm lip}$ is the cross-modal Lipschitz constant. Therefore, when $D_k < D_{\mathrm{global}}$, intra-expert imputation admits a strictly tighter worst-case reconstruction-error bound than global inter-expert imputation.

\begin{Remark}[Geometric Interpretation of Inter-Expert Shift]
For a reference representation $h_{v,g}$ routed to expert $j$ and a
target representation $h_{v,t}$ routed to expert $k$, their distance can
be decomposed as:

\begin{equation}
\label{eq:shift_error}
\begin{aligned}
    \left\| h_{v,g} - h_{v,t} \right\|_2
    &=
    \left\|
    (h_{v,g} - c_j)
    + (c_j - c_k)
    + (c_k - h_{v,t})
    \right\|_2 \\
    &\leq
    \underbrace{\|h_{v,g} - c_j\|_2}_{\text{Reference intra-cell deviation}}
    +
    \underbrace{\|c_j - c_k\|_2}_{\text{Expert centroid shift}}
    +
    \underbrace{\|c_k - h_{v,t}\|_2}_{\text{Target intra-cell deviation}}
\end{aligned}
\end{equation}

When the reference and target are routed to different experts ($j\neq k$), the additional centroid-separation term
$\|c_j-c_k\|_2$ may increase the geometric uncertainty of reconstruction. In contrast, intra-expert aggregation corresponds to $j=k$, such that $c_j=c_k$ and the centroid-shift term vanishes from the decomposition.
\end{Remark}

Although our theoretical analysis is developed for Sparse MoE, the proposed intra-expert imputation mechanism can also be instantiated with SoftMoE by restricting aggregation within the slots associated with the same expert. We therefore additionally evaluate this variant empirically to study whether the locality principle generalizes beyond discrete routing.

\begin{figure}
    \centering
    \includegraphics[width=\linewidth]{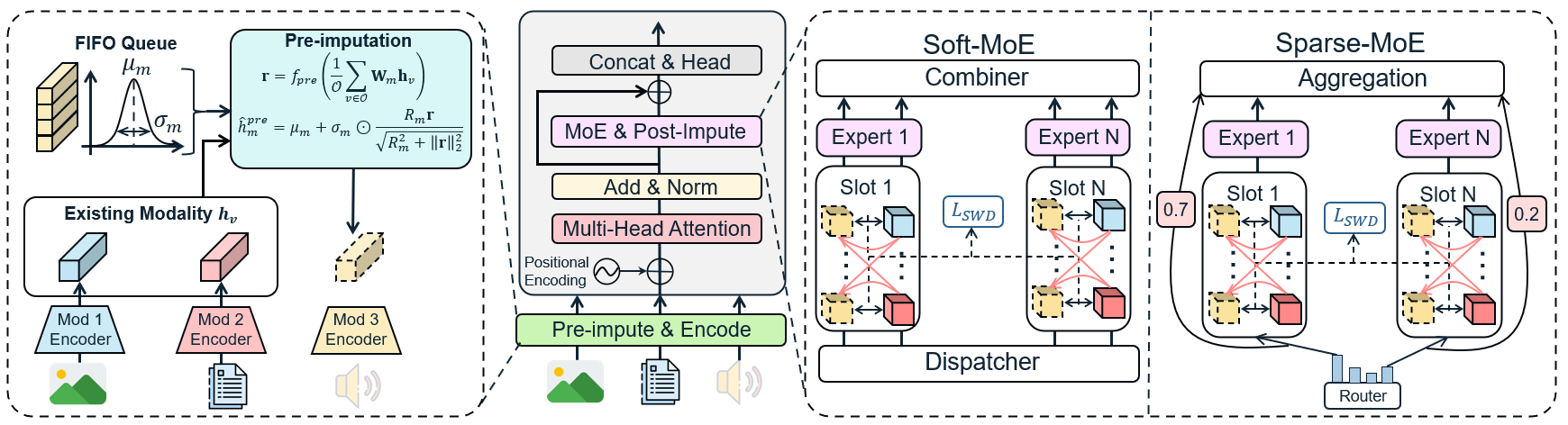}
    \caption{Overview of Proposed Method: (1) Starting at the encoding of modalities, where each modality will be assigned a unique encoder and implement pre-imputation if any modality is needed. Pre-imputation learns the imputed coefficient derived from existing modalities to impute the missing modality from its modality distribution. (2) Inspired by the theoretical motivation, post-imputation conducted intra-expert imputation before each token is forwarded through the expert. An SWD loss is applied to enforce the semantic alignment of tokens in the slots.}
    \label{fig:framework}
\end{figure}

\subsection{Pre-Imputation}
\label{sec:pre-imputate}

We formulate pre-imputation as a context-guided distributional shifting problem in the embedding space. Given an instance with a missing modality $h_m$ and a set of observed modalities $\mathcal{O}$, each available modality $h_v$ is first encoded into an embedding sequence $\mathbf{h}_v \in \mathbb{R}^{s \times d}$, where $v \in \mathcal{O}$, $s$ is the sequence length, and $d$ is the embedding dimension. Since the original embedding of the missing modality is unavailable, we do not aim to reconstruct an unknown instance-level target. Instead, we construct a context-dependent embedding $\hat{\mathbf{h}}^{\mathrm{pre}}_m$ that is constrained by the empirical distribution of modality $m$. For each modality $m$, we maintain a fixed-size FIFO queue $\mathcal{Q}_m$ containing recently observed encoder embedding sequences, providing modality-specific empirical mean $\mu_m$ and standard deviation $\sigma_m$. We estimate the centroid and standard deviation of the observed embedding bank as:

\begin{equation}
\boldsymbol{\mu}_m =
\frac{1}{|\mathcal{Q}_m|}
\sum_{h \in \mathcal{Q}_m}
h,
\qquad
\boldsymbol{\sigma}_m =
\sqrt{
\frac{1}{|\mathcal{Q}_m|}
\sum_{h \in \mathcal{Q}_m}
(h-\boldsymbol{\mu}_m)^2
}
\label{eq:buffer_statistics}
\end{equation}

We further characterize the robust extent of the observed modality distribution. Specifically, we compute the standardized distance of every observed token in the buffer and define $R_m$ as its $q$-th quantile:

\begin{equation}
R_m =
\operatorname{Quantile}_{q}
\left(
\left\{
\left\|
\frac{\mathbf{h}_{m,t}-\boldsymbol{\mu}_{m,t}}
{\boldsymbol{\sigma}_{m,t}}
\right\|_2
\;\middle|\;
\mathbf{h}_m \in \mathcal{Q}_m,\;
t=1,\ldots,s
\right\}
\right)
\label{eq:robust_radius}
\end{equation}

where $q=0.8$ in our experiments. The statistics $\boldsymbol{\mu}_m$ and $\boldsymbol{\sigma}_m$ describe the location and scale of the observed embedding distribution, whereas $R_m$ defines a robust standardized radius estimated from currently observed modality embeddings. To determine the imputation direction for the current instance, we aggregate the projected embeddings of its available modalities and predict an unconstrained residual $\mathbf{u}_m$ using a lightweight pre-imputation network $f_{\mathrm{pre}}$. The final pre-imputed embedding is obtained by radially soft-clipping this residual before shifting it from the empirical centroid:

\begin{equation}
\begin{aligned}
\mathbf{r} &=
f_{\mathrm{pre}}
\left( \frac{1}{|\mathcal{O}|}
\sum_{v\in\mathcal{O}} \mathbf{W}_m \mathbf{h}_v \right), 
\quad
\hat{\mathbf{h}}^{\mathrm{pre}}_m = \boldsymbol{\mu}_m + \boldsymbol{\sigma}_m \odot \frac{R_m\mathbf{r}}{\sqrt{R_m^2+\left\| \mathbf{r} \right\|_2^2}}
\label{eq:pre_imputation}
\end{aligned}
\end{equation}

The radial soft-clipping operation is applied token-wise. It retains the direction predicted from the available modalities while smoothly constraining the standardized residual magnitude below $R_m$. In contrast to an element-wise bounded activation such as $\tanh$, this operation does not independently saturate each embedding dimension. Consequently, the pre-imputed embedding can adapt to the current multimodal context without drifting arbitrarily far from the robust empirical geometry of the target modality. The resulting $\hat{\mathbf{h}}^{\mathrm{pre}}_m$ is directly passed to the MoE backbone; it is not decoded into raw modality data or processed again by the target modality encoder. The FIFO queue $\mathcal{O}_m$ is detached from backpropagation and updated only with embeddings from observed modalities during training. At inference time, the saved modality statistics are fixed, and no reference embedding from the unknown target modality is required.

\subsection{Post-Imputation}
\label{sec:post-imputate}
To refine the distribution-constrained initialization $\hat{\mathbf{h}}^{\mathrm{pre}}_m$ produced by the pre-imputation stage, we introduce an intra-expert post-imputation mechanism. Although pre-imputation constructs a plausible representation for the missing modality, it is performed before Sparse MoE routing and does not explicitly exploit the local semantic structure induced by individual experts. According to Theorem~\ref{thm:local_imputation}, restricting the reference and target representations to the same expert removes this inter-expert centroid shift and provides the motivation for our localized post-imputation design.

Since all modalities are processed by the same router, their representations are evaluated in a common expert-routing space. For a pre-imputed missing-modality representation routed to expert $k$, we only aggregate observed tokens from the same instance that are also routed to expert $k$. Consequently, both the missing representation and its reference context are associated with the same local expert region. Rather than performing global aggregation over tokens from geometrically distinct regions, we use a residual feature correction based on the average observed representation within the same expert. This design preserves the local expert opinion while avoiding unconstrained inter-expert information transfer. The theoretical analysis motivates where the reference representations should be selected from, rather than prescribing a specific aggregation operator. Let $\mathcal{R}_k = \{g | h_{v,g} \in \mathcal{V}_k\}$ denote the local set of observed tokens co-located in the $k$-th expert region. we instantiate the local aggregation using either averaging or cross-attention to produce convex combination weight $\sum_{g \in \mathcal{R}_k} w_g =1$ for post-imputation.

\begin{equation}
\bar{\mathbf{h}}_k =
 \sum_{g\in \mathcal{R}_k} w_g h_{v,g}
\qquad
\tilde{\mathbf{h}}^{k}_m =
\operatorname{LN}
\left(
\hat{\mathbf{h}}^{\mathrm{pre}}_m +
\bar{\mathbf{h}}_k
\right).
\label{eq:post_imputation}
\end{equation}

Here, $\bar{\mathbf{h}}_k$ represents the local observed context for expert $k$, and $\tilde{\mathbf{h}}^{(k)}_m$ is the representation subsequently processed by its expert FFN. This operation is applied token-wise. When no observed token is co-located with the missing token in an expert region, the model retains the pre-imputed representation without an additional reference correction. To ensure that post-imputed representations remain compatible with the observed distribution of each modality, we applied an unpaired loss function to controlled sliced Wasserstein distance (SWD) loss \cite{heitz2021sliced}. One advantage of SWD loss is that we don't need to assume the target distribution like KL divergence loss. Under unpaired condition, we have no observation on target embedding of missing modality, but we do observe the empirical distribution of the modality. Thus, we can only constraint the overall geometry of imputed modality instead of matching the exact embedding. For modality $a$, let $\mathcal{H}^{\mathrm{obs}}_a$ and $\mathcal{H}^{\mathrm{imp}}_a$ denote the sets of mean-pooled post-imputation representations from observed and missing samples, respectively. Given $L$ random normalized projection directions, let $Q_l(\mathcal{H})$ denote the empirical quantiles of the projected set $\mathcal{H}$ along the $l$-th direction. We define the modality-wise distribution matching loss as:

\begin{equation}
\mathcal{L}_{\mathrm{SWD}} =
\frac{1}{|\mathcal{M}|}
\sum_{m \in \mathcal{M}}
\frac{1}{L}
\sum_{l=1}^{L}
\left\|
Q_l\left(\mathcal{H}^{\mathrm{obs}}_a\right)
-
Q_l\left(\mathcal{H}^{\mathrm{imp}}_a\right)
\right\|_2^2,
\label{eq:swd_loss}
\end{equation}

where $\mathcal{M}$ contains the modalities with both observed and imputed samples in the current mini-batch. The observed representations $\mathcal{H}^{\mathrm{obs}}_a$ are treated as detached empirical samples, such that the SWD gradient only updates the imputed representations. This objective does not require paired observed and imputed samples; it aligns their modality-wise distributions through multiple one-dimensional projections.

\section{Experiment} \label{sec:experiment}

\begin{table}[h]
  \caption{Experiment setting I: Main results on MIMIC-III}
  \label{tb:mimic-3-experiment}
  \centering
  \resizebox{\textwidth}{!}{%
  \begin{tabular}{cc|cccccccc>{\columncolor{gray!20}}c>{\columncolor{gray!20}}c}
    \toprule
    Task   & Metric & SMIL & ShaSpec  & TF & LlMoE & FuseMoE-S & FuseMoE-L & FlexMoE & ConfSMoE & \textbf{TWIN-SF} & \textbf{TWIN-SP} \\
    \midrule
    \multirow{2}{*}{48-IHM} & F1  & 40.36 $\pm$ 2.10 & 38.06 $\pm$ 0.51 & 11.55 $\pm$ 0.65 & 45.08 $\pm$ 0.71 & 41.34 $\pm$ 0.33 & 42.70 $\pm$ 0.61 & 41.04 $\pm$ 0.96 & 44.12 $\pm$ 0.33 & 46.95 $\pm$ 0.81 & \textbf{49.82 $\pm$ 0.82}  \\
                            & AUC & 81.25 $\pm$ 0.87 & 79.16 $\pm$ 0.15 & 72.92 $\pm$ 0.55 & 83.96 $\pm$ 0.10 & 82.93 $\pm$ 0.76 & 81.82 $\pm$ 1.02 & 83.17 $\pm$ 0.46 & 83.53 $\pm$ 0.20 & 85.11 $\pm$ 0.38 & \textbf{85.56 $\pm$ 0.11}  \\
    \midrule
    \multirow{2}{*}{LOS}    & F1  & 64.87 $\pm$ 0.46 & 37.97 $\pm$ 2.18 & 49.87 $\pm$ 1.22 & 62.67 $\pm$ 0.89 & 62.64 $\pm$ 0.15 & 62.76 $\pm$ 0.13 & 61.26 $\pm$ 0.72 & 63.72 $\pm$ 0.24 & 65.93 $\pm$ 0.30 & \textbf{66.42 $\pm$ 0.49} \\
                            & AUC & 79.80 $\pm$ 0.56 & 77.17 $\pm$ 0.30 & 69.94 $\pm$ 0.57 & 79.71 $\pm$ 0.21 & 78.35 $\pm$ 0.33 & 78.31 $\pm$ 0.45 & 77.23 $\pm$ 0.69 & 80.27 $\pm$ 0.82 & \textbf{81.66 $\pm$ 0.38} & 81.49 $\pm$ 0.32  \\
    \midrule
    \multirow{2}{*}{25-PHE} & F1  & 28.70 $\pm$ 0.90 & 19.77 $\pm$ 0.72 & 5.19 $\pm$ 0.35  & 34.41 $\pm$ 0.28 & 31.16 $\pm$ 0.20 & 32.53 $\pm$ 0.90 & 36.61 $\pm$ 0.77 & 36.74 $\pm$ 1.59 & \textbf{38.63 $\pm$ 0.57} & 38.11 $\pm$ 0.45  \\
                            & AUC & 75.75 $\pm$ 0.26 & 71.66 $\pm$ 1.33 & 66.00 $\pm$ 0.12 & 78.19 $\pm$ 0.41 & 78.78 $\pm$ 0.16 & 78.55 $\pm$ 0.10 & 78.01 $\pm$ 0.11 & 78.52 $\pm$ 0.49 & \textbf{80.20 $\pm$ 0.10} & 79.81 $\pm$ 0.29  \\
    \bottomrule
  \end{tabular}
  }
\end{table}

\begin{table}[h]
  \caption{Experiment setting II: Main results on CMU-MOSEI}
  \label{tb:CMU-MOSEI_Experiment II}
  \centering
  \resizebox{\textwidth}{!}{%
  \begin{tabular}{cc|ccccccc>{\columncolor{gray!20}}c>{\columncolor{gray!20}}c}
    \toprule
    Missing Rate   & Metric & ShaSpec & TF & LlMoE & FuseMoE-S & FuseMoE-L & FlexMoE & ConfSMoE & \textbf{TWIN-SF} & \textbf{TWIN-SP}  \\
    \midrule
    \multirow{2}{*}{0\%}  & F1  & 47.48 $\pm$ 0.27 & 38.89 $\pm$ 0.54 & 59.13 $\pm$ 0.52 & 57.04 $\pm$ 0.15 & 58.71 $\pm$ 0.23 & 60.62 $\pm$ 0.96 & 61.31 $\pm$ 0.14 & 61.75 $\pm$ 0.13 & \textbf{62.17 $\pm$ 0.44} \\
                          & AUC & 67.93 $\pm$ 0.56 & 64.26 $\pm$ 0.61 & 80.29 $\pm$ 0.31 & 78.54 $\pm$ 0.42 & 78.13 $\pm$ 0.34 & 81.34 $\pm$ 0.20 & 82.87 $\pm$ 0.10 & 82.68 $\pm$ 0.16 & \textbf{83.48 $\pm$ 0.28} \\
    \midrule
    \multirow{2}{*}{10\%} & F1  & 47.41 $\pm$ 0.18 & 34.94 $\pm$ 1.10 & 57.46 $\pm$ 1.10 & 55.90 $\pm$ 0.55 & 56.68 $\pm$ 0.31 & 58.66 $\pm$ 1.10 & 59.81 $\pm$ 0.34 & 59.40 $\pm$ 0.66 & \textbf{60.65 $\pm$ 0.22} \\
                          & AUC & 69.26 $\pm$ 0.22 & 61.05 $\pm$ 0.34 & 79.38 $\pm$ 0.94 & 76.38 $\pm$ 0.39 & 76.44 $\pm$ 0.12 & 79.48 $\pm$ 0.23 & 81.00 $\pm$ 0.10 & 80.78 $\pm$ 0.15 & \textbf{81.90 $\pm$ 0.05} \\
    \midrule
    
    \multirow{2}{*}{20\%} & F1  & 46.54 $\pm$ 1.04 & 35.37 $\pm$ 0.94 & 55.35 $\pm$ 0.88 & 54.44 $\pm$ 0.27 & 54.60 $\pm$ 0.19 & 55.96 $\pm$ 1.67 & 58.06 $\pm$ 0.21 & 56.81 $\pm$ 0.60 & \textbf{58.29 $\pm$ 0.64} \\
                          & AUC & 68.05 $\pm$ 0.56 & 59.54 $\pm$ 0.22 & 77.29 $\pm$ 0.47 & 74.32 $\pm$ 0.16 & 74.92 $\pm$ 0.64 & 78.05 $\pm$ 0.10 & 78.85 $\pm$ 0.28 & 78.63 $\pm$ 0.28 & \textbf{78.73 $\pm$ 0.21} \\
    \midrule

    \multirow{2}{*}{30\%} & F1  & 45.05 $\pm$ 0.90 & 35.36 $\pm$ 0.84 & 53.93 $\pm$ 1.13 & 52.27 $\pm$ 0.35 & 49.92 $\pm$ 0.36 & 52.00 $\pm$ 0.47 & 54.16 $\pm$ 0.49 & 54.71 $\pm$ 0.90 & \textbf{56.16 $\pm$ 0.31} \\
                          & AUC & 66.89 $\pm$ 0.31 & 60.23 $\pm$ 0.31 & 76.50 $\pm$ 0.31 & 72.22 $\pm$ 0.52 & 72.52 $\pm$ 0.29 & 76.14 $\pm$ 0.38 & 76.91 $\pm$ 0.41 & 76.62 $\pm$ 0.41 & \textbf{77.04 $\pm$ 0.48} \\
    \midrule
    
    \multirow{2}{*}{40\%} & F1  & 43.49 $\pm$ 0.95 & 32.23 $\pm$ 0.61 & 50.95 $\pm$ 0.99 & 49.32 $\pm$ 0.17 & 49.36 $\pm$ 0.60 & 49.29 $\pm$ 0.28 & 50.72 $\pm$ 0.37 & 52.48 $\pm$ 0.85 & \textbf{52.34 $\pm$ 0.18} \\
                          & AUC & 65.32 $\pm$ 0.59 & 60.09 $\pm$ 0.35 & 73.24 $\pm$ 0.76 & 69.85 $\pm$ 0.11 & 70.05 $\pm$ 0.46 & 72.08 $\pm$ 1.27 & 73.63 $\pm$ 0.48 & 74.52 $\pm$ 0.53 & \textbf{74.40 $\pm$ 0.32} \\
    \midrule
    
    \multirow{2}{*}{50\%} & F1  & 42.75 $\pm$ 1.22 & 33.01 $\pm$ 0.51 & 46.69 $\pm$ 0.63 & 43.08 $\pm$ 0.47 & 44.44 $\pm$ 0.36 & 46.19 $\pm$ 0.34 & 47.31 $\pm$ 0.20 & 49.84 $\pm$ 0.80 & \textbf{49.49 $\pm$ 0.17} \\
                          & AUC & 62.32 $\pm$ 0.42 & 58.20 $\pm$ 0.17 & 70.90 $\pm$ 0.82 & 67.63 $\pm$ 0.55 & 67.44 $\pm$ 0.16 & 70.15 $\pm$ 1.10 & 70.53 $\pm$ 0.33 & 71.26 $\pm$ 0.34 & \textbf{71.31 $\pm$ 0.22} \\
    \bottomrule
  \end{tabular}
  }
\end{table}

\subsection{Primary Results}

\textbf{Experimental Setting I} evaluates all methods under the naturally occurring missing-modality patterns in the clinical datasets. As shown in Tables \ref{tb:mimic-3-experiment} and \ref{tb:mimic-4-experiment}, we evaluate TWIN on MIMIC-III and MIMIC-IV across three clinical prediction tasks: 48-IHM, LOS, and 25-PHE. Overall, TWIN consistently achieves the strongest performance across both datasets and different task settings. On MIMIC-III, TWIN improves over the strongest baseline by 4.74 F1 points and 1.60 AUC points on 48-IHM, while also achieving improvements of 1.55 and 1.89 F1 points on LOS and 25-PHE, respectively. Similar gains are observed on MIMIC-IV, where TWIN improves 48-IHM F1 from 43.76 to 47.05 and consistently achieves the best or comparable results on the remaining tasks. Notably, the best-performing backbone varies across tasks, with both TWIN-SF and TWIN-SP contributing the strongest results, suggesting that the proposed imputation framework is effective across different MoE routing mechanisms. These results demonstrate that TWIN remains robust under naturally occurring modality missingness in real-world clinical data.

\textbf{Experimental Setting II} randomly removes a fixed proportion of modalities during both training and testing while ensuring that at least one modality remains observed. As shown in Tables \ref{tb:CMU-MOSI_Experiment II} and \ref{tb:CMU-MOSEI_Experiment II}, we evaluate all methods under missing ratios ranging from 0\% to 50\%. As the missing ratio increases, the performance of all methods gradually degrades, while TWIN exhibits a more graceful decline and maintains a clear advantage under severe missingness. In particular, at a 50\% missing ratio, TWIN-SP achieves 46.62 F1 on CMU-MOSI compared with 44.34 for ConfSMoE, while TWIN-SF achieves the highest AUC of 71.03. A similar trend is observed on CMU-MOSEI, where TWIN-SF reaches 49.84 F1 and TWIN-SP achieves 71.31 AUC at the same missing ratio, outperforming the strongest competing baselines. These results indicate that the proposed two-stage imputation framework remains effective even when the available multimodal context becomes substantially limited. We also observe that the relative performance of the Sparse and Soft MoE variants varies across missing ratios and evaluation metrics, suggesting that TWIN is compatible with different routing mechanisms rather than relying on a specific MoE backbone.

\begin{figure}
    \centering
    \includegraphics[width=\linewidth]{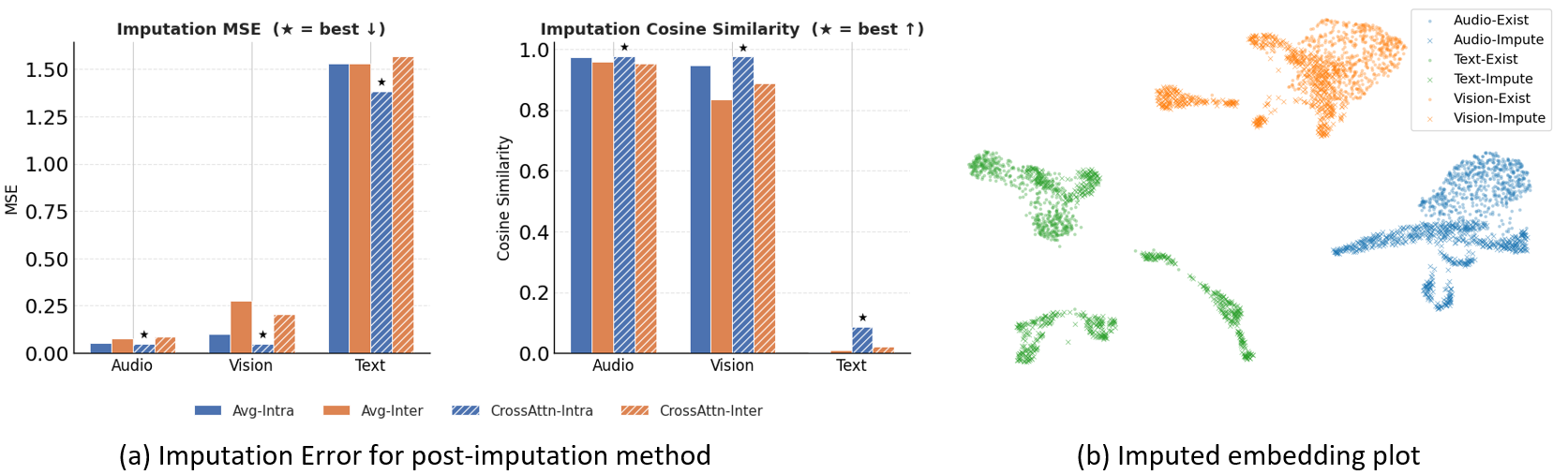}
    \caption{Imputation Error of different post-imputation method. We take two post imputation method cross-attention and average, and compare their combination of inter- and intra-expert imputation.}
    \label{fig:imputation_error}
\end{figure}
\textbf{Experimental Setting III} evaluates a more challenging missing-modality scenario by applying modality dropout during training and testing the model with only one or two available modalities. The results show that TWIN remains robust even when entire modalities are unavailable. On CMU-MOSI, TWIN achieves the best AUC and F1 performance across all evaluated modality combinations, with particularly large improvements when only limited modality information is available. For example, under the Video-Audio setting, TWIN improves AUC from 57.61 for the strongest baseline to 69.90, while also increasing F1 from 37.37 to 45.10. On CMU-MOSEI, TWIN achieves the best or comparable performance in nearly all modality configurations, demonstrating that the proposed imputation framework generalizes well across different missingness patterns. The advantage remains evident even in single-modality settings, indicating that TWIN can effectively exploit the limited observed information to construct useful representations for the missing modalities. Overall, these results demonstrate the robustness of TWIN under severe and asymmetric modality absence.

\subsection{How well our imputation can be}

\begin{wraptable}{r}{0.6\textwidth}
    \centering
    \caption{F1 score comparison of different backbones and imputation strategy on CMU-MOSI (Experiment Setting II, missing ratio $0.5$)}
    \resizebox{\linewidth}{!}{%
    \label{tab:post-imput}
    \begin{tabular}{c|cccc}
        \toprule
         \multirow{2}{*}{Post Imput.}    & \multicolumn{2}{c}{Sparse MoE} & \multicolumn{2}{c}{Soft MoE} \\
         
                       & Average  & Cross-Attention & Average & Cross-Attention \\
         \midrule 
         Inter-Expert  & 45.37 $\pm$ 2.09 & 45.41 $\pm$ 2.20 & 44.59 $\pm$ 2.37 & 44.55 $\pm$ 2.39  \\
         \midrule
         Intra-Expert  & 46.18 $\pm$ 3.23 & 46.26 $\pm$ 2.97 & 45.98 $\pm$ 1.70 & 46.35 $\pm$ 0.66 \\
         \bottomrule
    \end{tabular}}
\end{wraptable}
\textbf{Imputation Error}: The comparison in Figure \ref{fig:imputation_error} provides complementary evidence beyond downstream classification performance. Intra-expert aggregation consistently reduces reconstruction error relative to inter-expert aggregation, indicating that restricting reference tokens to the same routing region provides a cleaner reconstruction context. Cross-attention further improves the result because it does not treat all co-routed references equally, but adaptively assigns larger weights to references that are more relevant to the current missing representation. In contrast, averaging assumes uniform contribution from all available references within the region. The larger reconstruction error observed for text may reflect its higher semantic variability and the difficulty of recovering discrete linguistic information solely from audio-visual context. The UMAP visualization in Figure \ref{fig:imputation_error} (b) further suggests that the imputed representations remain close to their corresponding modality clusters rather than drifting toward unrelated modality regions. We interpret this observation as empirical support for the locality principle underlying TWIN, rather than as direct evidence of exact reconstruction of the unavailable modality.

\textbf{Different Approaches in Pre-Imputation}: As discussed in Section~\ref{sec:pre-imputate}, the pre-imputation module requires an activation function that constrains the estimated representation to remain consistent with the distribution of the target modality. To investigate this effect, we compare different activation functions and report their downstream performance in Table~\ref{tab:preimpute}. The results show that the prototype-based strategy adopted in our method consistently outperforms both bounded and unbounded activation functions. Moreover, most bounded activations achieve better performance than unbounded functions such as Arcsinh, suggesting that constraining the output range is beneficial for modality-aware pre-imputation. This is because bounded functions can restrict the imputed representation within a controlled distributional range, reducing the risk of drifting away from the target modality distribution during reconstruction.

\textbf{Different Approaches in Post-Imputation}: To study which configuration is optimal, we implement an empirical study on post-imputation, as shown in Table \ref{tab:post-imput}, we can clearly see that the intra-expert imputation consistently outperforms the inter-expert under different combination of post-imputation setting, while the cross-attention also wins average aggregation. This is because the adaptive aggregation weights that cross-attention can apply. These results empirically support the geometric motivation of Theorem \ref{thm:local_imputation}: restricting references to the same expert region generally reduces reconstruction error and improves downstream performance.

\subsection{Component Ablation and Compute-Controlled Efficiency Analysis}

\textbf{Component Ablation}: We implement ablation study on our designed component in Table \ref{tab:module_dropout}. The ablation study is implemented on CMU-MOSI missing ratio 50\% and  We can observe that pre- and post-imputation indeed impacts the downstream task. Specifically, we drop the pre imputation module and the performance drops in around 1\% to 2\% in both F1 and AUC in MIMIC-III. More importantly, once the pre- and post- imputation module work together, we can observe that the performance reaches the highest, which conclude that both module can produce positive synergy effect.

\textbf{Compute-Controlled Efficiency Analysis}: Table~\ref{tab:complexity} compares model efficiency under approximately matched computational budgets. Across small, medium, and large settings, TWIN maintains MFLOPs comparable to FlexMoE and ConfSMoE while using substantially less parameters. For example, in the medium setting, TWIN uses only about 2.8M parameters compared with the competing MoE with 11M parameter methods under the similar MFLOPs, while TWIN-SP achieves the highest F1 score of 45.78. A similar trend is observed in the large setting, where TWIN-SP reaches 47.47 F1 with about 13.9M parameters, compared with over 33M parameters for FlexMoE and ConfSMoE. These results demonstrate that TWIN provides a favorable trade-off between computational cost, parameter efficiency, and predictive performance. LIMoE demonstrated different complexity scale to other baselines because we found it is difficult to make it small without changes of architecture significantly. But it is still not comparable to our proposed method.

\begin{table}[ht]
\caption{Module Ablation on MIMIC-III IHM task and CMU-MOSI (Experiment Setting II, missing ratio $0.5$)}
\label{tab:module_dropout}
\centering
\resizebox{\linewidth}{!}{%
\begin{tabular}{c|c|cccc|cccc}
  \toprule
   \multirow{2}{*}{Variants}    & \multirow{2}{*}{Metrics}   & \multicolumn{4}{c|}{Sparse MoE} & \multicolumn{4}{c}{Soft MoE} \\
                                &  & w/o pre & w/o post & w/o impute. & TWIN  & w/o pre & w/o post & w/o impute. & TWIN  \\
    \midrule
  \multirow{2}{*}{MIMIC-III}    & F1    & 48.24 $\pm$ 0.38 & 48.65 $\pm$ 0.25 & 44.90 $\pm$ 1.61 & 49.82 $\pm$ 0.82 & 46.28 $\pm$ 0.99 & 46.01 $\pm$ 0.60 & 44.78 $\pm$ 1.09 & 46.95 $\pm$ 0.81 \\
                                & AUC   & 84.37 $\pm$ 0.26 & 85.38 $\pm$ 0.28 & 84.30 $\pm$ 0.84 & 85.56 $\pm$ 0.11 & 84.06 $\pm$ 0.75 & 83.89 $\pm$ 0.39 & 83.90 $\pm$ 0.26 & 85.11 $\pm$ 0.38 \\
    \midrule
  \multirow{2}{*}{CMU-MOSI} 
                                & F1     & 45.59 $\pm$ 1.83 & 45.01 $\pm$ 2.03 & 45.01 $\pm$ 0.16 & 46.62 $\pm$ 4.43 & 44.42 $\pm$ 2.02 & 45.25 $\pm$ 3.33 & 44.03 $\pm$ 1.89 & 44.29 $\pm$ 1.92 \\
                                & AUC    & 69.94 $\pm$ 0.75 & 70.30 $\pm$ 1.52 & 69.39 $\pm$ 0.31 & 70.50 $\pm$ 1.30 & 69.59 $\pm$ 0.45 & 70.47 $\pm$ 0.39 & 67.86 $\pm$ 0.80 & 71.03 $\pm$ 2.43 \\
  \bottomrule
\end{tabular}
}
\end{table}

\begin{table}[h] 
    \centering
    \caption{Control Complexity Study: We control the complexity of the model to evaluate performance}
    \label{tab:complexity}
    \resizebox{\linewidth}{!}{%
    \begin{tabular}{c|c|cccccc>{\columncolor{gray!20}}c>{\columncolor{gray!20}}c}
        \toprule
         Scale & Metrics   & ShaSpec & TF & LlMoE & FuseMoE & FlexMoE & ConfSMoE & \textbf{TWIN-SF} & \textbf{TWIN-SP} \\
         \midrule 
         \multirow{3}{*}{Small}              & MFLOPs    & 21.73 & 39.69 & 1245.03 & 30.18 & 40.32 & 39.62 & 39.34 & 39.88 \\
         & \#Params                          & 2,177,406 & 9,921,903 & 2,130,227 & 9,636,174 & 2,184,963 & 2,184,730 & 335,691 & 336.376 \\
         & F1 score                          & 39.12 $\pm$ 1.21 & 35.81 $\pm$ 0.85 & 41.11 $\pm$ 0.90 & 33.85 $\pm$ 0.76 & 42.12 $\pm$ 1.18 & 44.34 $\pm$ 0.77 & 43.69 $\pm$ 0.60 & 44.72 $\pm$ 0.73 \\
         \midrule 
         \multirow{3}{*}{Medium} & MFLOPs    & 111.27 & 263.53 & 6803,84 & 286.02 & 263.94 & 268.62 & 267.04 & 267.45 \\
         & \#Params                          & 11,137,046 & 65,878,703 & 11,469,437 & 11,524,392 & 11,371,473 & 11,330,062 & 2,789,275 & 2,794,759  \\
         & F1 score                          & 41.56 $\pm$ 0.67 & 38.62 $\pm$ 0.38 & 43.21 $\pm$ 0.56 & 39.44 $\pm$ 0.81 & 43.78 $\pm$ 0.36 & 45.37 $\pm$ 1.73 &  44.88 $\pm$ 1.10 &  45.78 $\pm$ 2.22 \\
         \midrule
         \multirow{3}{*}{Large} &  MFLOPs    & 331.87           & 1021.15           & 19668.02         & 1032.81           & 1181.48          & 1078.82          & 1054.20    &  1064.16  \\
                                & \#Params   & 33,203,156       & 255,274,503       & 33,000,827       & 378,252,012       & 33,385,803       & 33,398,654       & 13,722,627 & 13,864,291   \\
                                & F1 score   & 42.82 $\pm$ 0.74 & 37.64 $\pm$ 0.88 & 43.45 $\pm$ 0.82 & 41.62 $\pm$ 0.84 & 43.46 $\pm$ 1.36 & 45.55 $\pm$ 1.69 & 45.63 $\pm$ 1.18 & 47.47 $\pm$ 1.67 \\
         \bottomrule
    \end{tabular}}
    \label{tab:complexity on CMU-MOSI}
\end{table}

\section{Conclusion}
In this work, we propose TWIN, a two-stage intra-expert imputation framework for robust multimodal learning with missing modalities. TWIN combines distribution-aware pre-imputation with locality-aware post-imputation and can be integrated with both Sparse and Soft MoE backbones. Our geometric analysis shows that, when expert-specific routing regions are more compact than the global representation space, intra-expert aggregation admits a tighter worst-case reconstruction-error bound than unconstrained global aggregation. Experiments across four multimodal benchmarks and multiple missingness settings demonstrate that TWIN consistently achieves state-of-the-art performance while maintaining strong parameter efficiency. These results highlight the effectiveness of exploiting expert-local representation structure for incomplete multimodal learning.

\section{AI disclosure}
In this work, we used generative AI tools for assist in the writing of proofs, refine hypotheses, provide feedback on research  methodology or experiments, assist with translation. We have not used generative AI tools for help develop theoretical models or conceptual frameworks, formulate mathematical claims, propose hypotheses, design research  methodology or experiments, implement methods, clean and reformat dataset, interpret results, and support qualitative and thematic data analysis,  are not applicable to this work. Additionally, we used generative AI tools for editing the research paper to improve readability and identify relevant literature, which all AI generated content have been checked by more than two authors.. We have reviewed all AI-assisted work. We take responsibility for the final content of this work, including text, claims or artifacts produced with the aid of generative AI.

\bibliography{iclr2027_conference}
\bibliographystyle{iclr2027_conference}

\newpage
\appendix

\section{Structure of SoftMoE}
Specifically, given an input data sequence $\mathbf{x} \in \mathbb{R}^{s \times d}$, SoftMoE introduces an \textit{expert-choice} routing paradigm, operating on $N$ experts with $p$ slots per expert. Formally, the routing mechanism is parameterized by learnable slot parameters $\Phi \in \mathbb{R}^{d \times N \times p}$, where $\Phi_{j, k} \in \mathbb{R}^d$ denotes the parameter vector for the $k$-th slot of the $j$-th expert. To select tokens from the perspective of experts, SoftMoE computes dispatch weights $D_{i, j, k}$ by applying a softmax over the sequence dimension $s$:

\begin{equation} 
    D_{i, j, k} = \frac{\exp(\mathbf{x}_i \cdot \Phi_{j, k})}{\sum_{i'=1}^{s} \exp(\mathbf{x}_{i'} \cdot \Phi_{j, k})}, \quad \tilde{\mathbf{x}}_{j, k} = \sum_{i=1}^{s} D_{i, j, k} \mathbf{x}_i
    \label{eq:dispatch_prob}
\end{equation}

where $i \in \{1, \dots, s\}$ is the token index, $j \in \{1, \dots, N\}$ is the expert index, and $k \in \{1, \dots, p\}$ is the slot index. Each slot $(j, k)$ then extracts a soft representation $\tilde{\mathbf{x}}_{j, k} \in \mathbb{R}^d$ by performing a probabilistically weighted sum of all input tokens. This expert-choice mechanism bypasses the discrete top-$K$ hard masking of Sparse MoE, enabling a fully differentiable, dense selection of tokens. Next, each expert $f_j$ processes its assigned slot inputs independently. Since SoftMoE fully activates all experts, the output for each slot $\mathbf{y}_{j, k} \in \mathbb{R}^d$ is computed as:

\begin{equation}
    \mathbf{y}_{j, k} = f_j(\tilde{\mathbf{x}}_{j, k}), \quad \forall j \in \{1, \dots, N\}, k \in \{1, \dots, p\}
\end{equation}

where $f_j(\cdot)$ represents the Feed-Forward Network (FFN) block associated with the $j$-th expert. To reconstruct the final output sequence, these slot-level outputs must be aggregated back to the token level. Analogous to the dispatch step, combine weights $C \in \mathbb{R}^{s \times N \times p}$ are calculated by applying a softmax across the expert-slot dimensions ($N \times p$) for each token. Finally, the token-level output $\mathbf{z}_i \in \mathbb{R}^d$ is formulated as a weighted sum of all processed slot representations:

\begin{equation}
    C_{i, j, k} = \frac{\exp(\mathbf{x}_i \cdot \Phi_{j, k})}{\sum_{j'=1}^{N} \sum_{k'=1}^{p} \exp(\mathbf{x}_i \cdot \Phi_{j', k'})}, \quad \mathbf{z}_i = \sum_{j=1}^N \sum_{k=1}^p C_{i, j, k} \mathbf{y}_{j, k}
    \label{eq:combine_prob}
\end{equation}

This dual-softmax mechanism ensures that the information flow remains balanced and fully differentiable, establishing a continuous global aggregation counterpart to the localized discrete boundaries of Sparse MoE.

\section{Additional Theoretical Contribution}

\subsection{Proof of Proposition~\ref{prop:voronoi_partition}}
\label{app:voronoi_proof}

\begin{proof}
Consider a token $h\in\mathcal{H}$. The routing function assigns $h$ to expert $k$ when its routing logit is no smaller than that of any other expert:
\begin{equation}
    h^\top c_k + b_k
    \geq
    h^\top c_j + b_j,
    \qquad
    \forall j\neq k.
\end{equation}

From a geometric perspective, nearest-centroid assignment associates $h$ with centroid $c_k$ when
\begin{equation}
    \|h-c_k\|_2^2
    \leq
    \|h-c_j\|_2^2,
    \qquad
    \forall j\neq k.
\end{equation}

Expanding the squared Euclidean distances gives
\begin{align}
    \|h-c_k\|_2^2
    &\leq
    \|h-c_j\|_2^2 \nonumber\\
    \Longleftrightarrow\quad
    \|h\|_2^2-2h^\top c_k+\|c_k\|_2^2
    &\leq
    \|h\|_2^2-2h^\top c_j+\|c_j\|_2^2 \nonumber\\
    \Longleftrightarrow\quad
    h^\top c_k-\frac{1}{2}\|c_k\|_2^2
    &\geq
    h^\top c_j-\frac{1}{2}\|c_j\|_2^2.
\end{align}

Therefore, by parameterizing the routing bias as
\begin{equation}
    b_k=-\frac{1}{2}\|c_k\|_2^2,
\end{equation}
maximizing the routing logit is exactly equivalent to minimizing the squared Euclidean distance to the routing centroids:
\begin{equation}
    f_{\mathcal V}(h)
    =
    \arg\max_k
    \left(
    h^\top c_k-\frac{1}{2}\|c_k\|_2^2
    \right)
    =
    \arg\min_k
    \|h-c_k\|_2^2.
\end{equation}

Consequently, the routing region associated with expert $k$ coincides with the corresponding Voronoi region restricted to the feature space $\mathcal{H}$:
\begin{equation}
    \mathcal V_k
    =
    \left\{
    h\in\mathcal H
    \mid
    \|h-c_k\|_2
    \leq
    \|h-c_j\|_2,
    \ \forall j\neq k
    \right\}.
\end{equation}
The interiors of the Voronoi cells are disjoint, while adjacent cells may share boundary points corresponding to routing ties. Since $\mathcal V_k\subseteq\mathcal H$, its diameter satisfies $D_k\leq D_{\mathrm{global}}$. This establishes the geometric partition used in our subsequent analysis.
\end{proof}

\subsection{Proof of Theorem~\ref{thm:local_imputation}}
\label{app:theoretical_proof}

We first introduce two mild assumptions required for the theoretical
analysis.

\paragraph{Assumption 1 (Bounded Cross-Modal Alignment).}
Let $h_{v,g}$ denote the available-modality representation of instance
$g$, and let $h_g^{*}$ denote its corresponding ideal representation in
the missing modality. We assume that the residual cross-modal discrepancy
is bounded by
\begin{equation}
\|h_{v,g}-h_g^{*}\|_2
\leq
\epsilon_{\mathrm{align}},
\qquad \forall g,
\label{eq:alignment_assumption}
\end{equation}
where $\epsilon_{\mathrm{align}}\geq0$ measures the residual misalignment
between the two modalities in the shared representation space.

\paragraph{Assumption 2 (Cross-Modal Lipschitz Continuity).}
Let $t$ denote the target missing-modality token and $g$ index candidate observed reference tokens within the same multimodal instance, we assume that the distance between their ideal missing-modality representations is bounded by the distance between their available-modality representations:
\begin{equation}
\|h_g^{*}-h_t^{*}\|_2
\leq
L_{\rm lip}\|h_{v,g}-h_{v,t}\|_2,
\label{eq:lipschitz_assumption}
\end{equation}
where $L_{\rm lip}>0$ is a Lipschitz constant.

\begin{proof}

Under intra-expert imputation, the missing representation of the target
instance $t$ is reconstructed by a convex combination of reference
representations routed to the same Voronoi cell $V_k$:
\begin{equation}
\hat{h}_m = \sum_{g\in \mathcal{R}_k} w_g h_{v,g},
\qquad
w_g \geq 0,
\qquad
\sum_{g\in \mathcal{R}_k}w_g=1.
\label{eq:appendix_local_imputation}
\end{equation}

The reconstruction error is therefore
\begin{align}
E_{\mathrm{local}}
&=
\|\hat{h}_m-h_t^{*}\|_2
\nonumber\\
&=
\left\|
\sum_{g\in \mathcal{R}_k}w_gh_{v,g} - h_t^{*}
\right\|_2.
\end{align}

Since the aggregation weights form a convex combination,
$\sum_{g\in \mathcal{R}_k}w_g=1$, we have
\begin{align}
E_{\mathrm{local}}
&=
\left\|
\sum_{g\in \mathcal{R}_k}
w_g
\left(
h_{v,g}-h_t^{*}
\right)
\right\|_2
\nonumber\\
&\leq
\sum_{g\in \mathcal{R}_k}
w_g
\|h_{v,g}-h_t^{*}\|_2,
\label{eq:triangle_local}
\end{align}
where the inequality follows from the triangle inequality.

For each reference instance $g$, we introduce its corresponding ideal
missing-modality representation $h_g^{*}$. Then,
\begin{align}
\|h_{v,g}-h_t^{*}\|_2
&=
\|
(h_{v,g}-h_g^{*})
+
(h_g^{*}-h_t^{*})
\|_2
\nonumber\\
&\leq
\|h_{v,g}-h_g^{*}\|_2
+
\|h_g^{*}-h_t^{*}\|_2.
\label{eq:error_decomposition}
\end{align}

Applying Assumptions~1 and~2 gives
\begin{align}
\|h_{v,g}-h_t^{*}\|_2
&\leq
\epsilon_{\mathrm{align}}
+
L_{\rm lip}\|h_{v,g}-h_{v,t}\|_2.
\label{eq:alignment_lipschitz}
\end{align}

Under the Voronoi routing mechanism, the target representation
$h_{v,t}$ and all intra-expert references $h_{v,g}$ are routed to the
same expert $k$. Therefore,
\begin{equation}
h_{v,t},h_{v,g}\in V_k.
\end{equation}

By definition of the diameter of the $k$-th Voronoi cell,
\begin{equation}
D_k
=
\sup_{h_a,h_b\in V_k}
\|h_a - h_b\|_2,
\end{equation}
and consequently,
\begin{equation}
\|h_{v,g}-h_{v,t}\|_2
\leq
D_k.
\label{eq:cell_bound}
\end{equation}

Combining Eqs.~\eqref{eq:alignment_lipschitz} and
\eqref{eq:cell_bound}, we obtain
\begin{equation}
\|h_{v,g}-h_t^{*}\|_2
\leq
\epsilon_{\mathrm{align}}
+
L_{\rm lip} D_k.
\end{equation}

Substituting this result into Eq.~\eqref{eq:triangle_local} gives
\begin{align}
E_{\mathrm{local}}
&\leq
\sum_{g\in \mathcal{R}_k}
w_g
\left(
\epsilon_{\mathrm{align}}
+
L_{\rm lip} D_k
\right)
\nonumber\\
&=
\left(
\epsilon_{\mathrm{align}}
+
L_{\rm lip} D_k
\right)
\sum_{g\in \mathcal{R}_k}w_g
\nonumber\\
&=
\epsilon_{\mathrm{align}}
+
L_{\rm lip} D_k.
\end{align}

Hence,
\begin{equation}
\boxed{
E_{\mathrm{local}}
\leq
\epsilon_{\mathrm{align}}
+
L_{\rm lip} D_k.
}
\end{equation}

For unconstrained global imputation, reference representations can be
selected from the entire representation space $\mathcal{H}$. Let
\begin{equation}
D_{\mathrm{global}} = \sup_{h_a,h_b\in\mathcal{H}} \|h_a-h_b\|_2.
\end{equation}

Following the same derivation, the reference-target distance is bounded
only by
\begin{equation}
\|h_{v,g}-h_{v,t}\|_2 \leq D_{\mathrm{global}},
\end{equation}
which leads to
\begin{equation}
\boxed{
E_{\mathrm{global}}
\leq
\epsilon_{\mathrm{align}}
+
L_{\rm lip} D_{\mathrm{global}}.
}
\end{equation}

Since $V_k\subseteq\mathcal{H}$, we have $D_k\leq D_{\mathrm{global}}$. Furthermore, when the expert-specific
Voronoi cell is strictly more compact than the global representation
space, i.e.,
\begin{equation}
D_k<D_{\mathrm{global}},
\end{equation}
it follows that
\begin{equation}
\epsilon_{\mathrm{align}}+L_{\rm lip} D_k < \epsilon_{\mathrm{align}}+L_{\rm lip} D_{\mathrm{global}}.
\end{equation}

Therefore, intra-expert imputation admits a strictly tighter
worst-case reconstruction-error upper bound than unconstrained global
imputation.

\end{proof}

\section{Experiment Details}

\subsection{Dataset}
\textbf{MIMIC-III}~\cite{johnson2016mimic} contains irregular clinical time series and clinical notes. We construct multimodal patient samples based on the preprocessing pipelines in~\cite{harutyunyan2019multitask,zhang2023improving}. Unlike~\cite{zhang2023improving}, which removes samples with incomplete modalities, we modify the pipeline to retain naturally missing cases. \textbf{MIMIC-IV}~\cite{johnson2023mimic} includes irregular time series, clinical notes, ECG signals, and chest X-rays. We adopt the MedFuse~\cite{hayat2022medfuse} preprocessing and dataset split for time-series data, while processing the remaining modalities separately to construct aligned patient samples. \textbf{CMU-MOSI}~\cite{zadeh2016mosi} is a multimodal sentiment dataset consisting of text, audio, and visual modalities. \textbf{CMU-MOSEI}~\cite{bagher2018multimodal} extends MOSI to a larger collection of YouTube videos with the same three modalities. For both datasets, we use the preprocessing pipeline from MMML~\cite{mustafa2022multimodal}. Following~\cite{zheng2025rethinking}, all methods are evaluated using consistent dataset splits and modality encoders. For MIMIC-III and MIMIC-IV, we use ClinicalBERT~\cite{alsentzer2019publicly} for clinical notes, single-layer patch embeddings for irregular time series and ECG, and a pretrained DenseNet from TorchXRayVision~\cite{cohen2022torchxrayvision} for chest X-rays. We consider three missing-modality settings: \textbf{Setting I} preserves natural missingness; \textbf{Setting II} randomly drops modalities during both training and testing while retaining at least one observed modality; and \textbf{Setting III} applies modality dropout during training and evaluates samples with only one or two available modalities. All experiments are repeated with three random seeds under the same hardware environment.

\subsection{Baselines}
We compare our method against representative approaches for multimodal learning and incomplete-modality scenarios. \textbf{TF}\cite{zadeh2017tensor} captures cross-modal interactions through modality-specific subnetworks followed by a tensor fusion layer. \textbf{LlMoE}\cite{mustafa2022multimodal} incorporates contrastive learning to facilitate multimodal representation learning. For methods explicitly designed for missing modalities, \textbf{SMIL}\cite{ma2021smil} adopts a Bayesian meta-learning strategy to reconstruct unavailable modality information, while \textbf{ShaSpec}\cite{wang2023multi} disentangles modality-shared and modality-specific representations through multiple training objectives and uses the shared representation to compensate for missing inputs.

We further include several MoE-based approaches. \textbf{FuseMoE}\cite{han2024fusemoe} jointly models samples with different modality availability patterns and employs a Laplacian-based gating mechanism for multimodal fusion and expert routing. We evaluate both its softmax-gated and Laplacian-gated variants, denoted as FuseMoE-S and FuseMoE-L, respectively. \textbf{FlexMoE}\cite{yun2024flex} associates different modality combinations with specialized experts and reconstructs absent modalities using a modality bank. Finally, \textbf{ConfSMoE}\cite{zheng2025rethinking} combines a two-stage modality imputation procedure with confidence-aware expert routing through ConfNet to improve robustness under incomplete multimodal inputs.

\subsection{Dataset and Task Description} \label{ap:dataset}

\begin{table}[h!]
\centering
\caption{Dataset statistics}
\resizebox{\linewidth}{!}{%
\begin{tabular}{c|cccccc}
    \toprule
    \textbf{Dataset} & \textbf{Task} & \textbf{\#Samples} & \textbf{Modality} & \textbf{\#Classes} & \textbf{Is Imbalanced?} & \textbf{Missing Ratio} \\
    \midrule
    \multirow{3}{*}{MIMIC-III}        & 48-IHM & 21,139   & Time Series, Text & 2 & Yes & 23.60\% \\
                                      & LOS    & 21,139   & Time Series, Text & 2 & Yes & 23.60\% \\
                                      & 25-PHE & 41,902 & Time Series, Text & 25 & Yes & 23.70\% \\
    \midrule 
    \multirow{3}{*}{MIMIC-IV}         & 48-IHM & 22,733 & Image, ECG, Time Series, Text & 2 & Yes & 47.10\% \\
                                      & LOS    & 22,733 & Image, ECG, Time Series, Text & 2 & Yes & 47.10\% \\
                                      & 25-PHE & 47,027 & Image, ECG, Time Series, Text & 25 & Yes & 50.32\% \\
    \midrule 
    CMU-MOSI         & Sentiment Analysis & 1870   & Audio, Video, Text & 3 & Yes & 0\% - 50\% \\
    CMU-MOSEI        & Sentiment Analysis & 20,985 & Audio, Video, Text & 3 & Yes & 0\% - 50\% \\
    \bottomrule
\end{tabular}}
\end{table}

\textbf{48-IHM} is a binary classification task for predicting in-hospital mortality. Following the benchmark setting, only the first 48 hours of ICU time-series observations are used as input.

\textbf{LOS} is formulated as a binary classification task under a similar setting to 48-IHM. We retain patients who have stayed in the ICU for at least 48 hours and use their observed records to predict the remaining length of stay.

\textbf{25-PHE} is a multi-label phenotype classification task that predicts the presence of 25 acute care conditions from each ICU stay. The definitions of these clinical conditions follow the MIMIC benchmark \cite{johnson2016mimic}.

\textbf{CMU-MOSI} and \textbf{CMU-MOSEI} are commonly used multimodal sentiment analysis datasets containing aligned text, audio, and video information. Each video segment is associated with a sentiment annotation. Since these datasets are originally complete across modalities, we simulate missing-modality scenarios by randomly dropping modalities during training.

All datasets considered in this study are publicly available and have been widely used in multimodal learning research. The datasets and corresponding preprocessing benchmarks can be accessed from the following resources:
\begin{itemize}
    \item CMU-MOSI and CMU-MOSEI: \textcolor{blue}{https://github.com/zehuiwu/MMML}
    \item MIMIC-III: \textcolor{blue}{https://physionet.org/content/mimiciii/1.4/}, with benchmark implementation at \textcolor{blue}{https://github.com/YerevaNN/mimic3-benchmarks}
    \item MIMIC-IV: \textcolor{blue}{https://physionet.org/content/mimiciv/2.2/}, with benchmark implementation at \textcolor{blue}{https://github.com/nyuad-cai/MedFuse}
\end{itemize}

\subsection{Hyperparameter and Hardware Setting}

All experiments are conducted in a remote server with NVIDIA A100 40 GB GPU.

\begin{table}[h!]
\centering
\caption{Hyperparameter setting}
\label{tb:hyperparameter}
\resizebox{0.35\linewidth}{!}{%
\begin{tabular}{c|cccccc}
    \toprule
    Parameter & Value \\
    \midrule
     \#Expert          & 8 \\
     \#Top-K Expert    & 1 \\
     Learning Rate     & 3e-4 \\
     Hidden Size       & 128 \\
     \#MoE layers      & 1 \\
     Training Epoch    & 50 \\
     Weight for $\mathcal{L}_{SWD}$ & 0.1 \\
     Random Seed       & 2025 \\
     Dropout           & 0.1 \\
    \bottomrule
\end{tabular}}
\end{table}

\subsection{Parameter Sensitivity}

We further conduct a sensitivity analysis to evaluate the robustness of the proposed method to several key hyperparameters, including the buffer quantile, FIFO queue capacity, and SWD loss weight. Overall, the model remains relatively stable across a broad range of settings, indicating that its performance is not overly sensitive to specific hyperparameter choices.

\begin{figure}
    \centering
    \includegraphics[width=\linewidth]{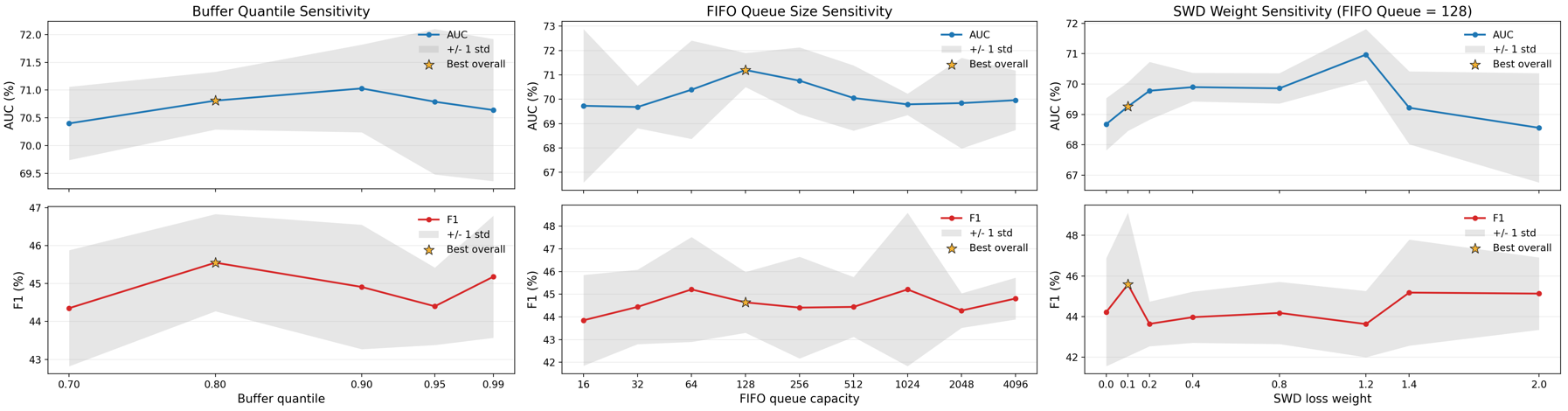}
    \caption{Ablation study on parameter sensitivity}
    \label{fig:addition_ablation_study}
\end{figure}

\textbf{Buffer quantile}. The model is relatively stable across different buffer quantiles, with only moderate variation in both AUC and F1. A quantile of 0.8 provides the best overall balance, achieving the highest F1 while maintaining competitive AUC. Increasing the quantile beyond 0.9 does not bring further improvements.

\textbf{FIFO queue size}.  Performance improves as the queue size increases from very small values, with 128 giving the best overall performance and the highest AUC. Larger queues provide little additional benefit and occasionally reduce performance, suggesting that a moderate queue size is sufficient to maintain representative historical embeddings.

\textbf{SWD loss weight.} The SWD weight shows a clearer trade-off between AUC and F1. A small weight of 0.1 gives the best overall balance, particularly improving F1, while larger weights can improve AUC but do not consistently benefit F1. This suggests that a relatively weak SWD regularization is sufficient, whereas excessive weighting may over-constrain the learned representations.

\section{Additional Pre-Imputation Activation}

To study the impact of pre-imputation, we implement different activation used in pre-imputation with bounded and unbounded activation functions as shown in Table \ref{tab:preimpute}. Whether our proposed pre-imputation module makes different to basic activation function that constrains the pre-imputed representation to an empirically calibrated standardized-radius region. Let's replace the soft-clipping term in \ref{eq:pre_imputation} to $\Delta_m$, indicating a sampling parameter from emperical distribution of missing modality. We provide additional ablation on different pre-imputation method and defined all pre-imputation baseline methods as below:

\textbf{Arcsinh} does not constrain the magnitude of the predicted residual according to the empirical embedding scale of the missing modality. It may therefore generate outlying pre-imputed representations.

\begin{equation}
    \Delta_m = Arcsinh (\textbf{W}_m \textbf{h}_v)
\end{equation}

\textbf{Tanh} bounds each embedding dimension independently to $[-1, 1]$, but this coordinate-wise constraint does not preserve the joint geometry, cross-dimensional correlations, or cluster structure of the empirical modality distribution. Its saturation also weakens gradients for large residuals.

\begin{equation}
    \Delta_m = tanh(\textbf{W}_m \textbf{h}_v)
\end{equation}

\textbf{Softsign} is likewise a coordinate-wise bounded activation and lacks any modality-specific empirical calibration. Although its gradients decay more gradually than those of Tanh, they still approach zero for large inputs and cannot constrain the global embedding radius.

\begin{equation}
    \Delta_m = softsign(\textbf{W}_m \textbf{h}_v), \quad softsign(x) = \frac{x}{1+|x|}
\end{equation}

\textbf{Gated} formulation adaptively suppresses each residual dimension but remains a coordinate-wise constraint without modeling the joint empirical geometry of the target modality. Its sigmoid gate can over-suppress the residual, leading to low-information pre-imputed embeddings or residual collapse. Let's denote two separate weight $\textbf{W}_g, \textbf{W}_t \in \mathbb{R}^{d \times d}$.

\begin{equation}
    \Delta_m = Sigmoid (\textbf{W}_g \textbf{h}_v) \odot tanh(\textbf{W}_t \textbf{h}_v)
\end{equation}

\section{Limitation}
The proposed method is designed for robust representation recovery rather than exact reconstruction of the missing modality as in generative approaches, and the imputed representation should therefore be interpreted as a task-oriented estimate instead of a faithful reconstruction of the original modality. Moreover, our theoretical analysis is primarily developed for Sparse MoE, while its extension to SoftMoE is validated empirically. The effectiveness of intra-expert imputation may also depend on the quality of the learned routing and cross-modal alignment.

\section{Additional Results}
These experiments extend the main results presented in Section Experiment and highlight the scalability of our method to real-world multimodal challenges.

\begin{table*}[h]
  \caption{Experiment setting I: Main results on MIMIC-IV}
  \label{tb:mimic-4-experiment}
  \centering
  \resizebox{\textwidth}{!}{%
  \begin{tabular}{cc|cccccccc>{\columncolor{gray!20}}c>{\columncolor{gray!20}}c}
    \toprule
    Task   & Metric & SMIL & ShaSpec & TF & LlMoE & FuseMoE-S & FuseMoE-L & FlexMoE & ConfSMoE & \textbf{TWIN-SF} & \textbf{TWIN-SP}  \\
    \midrule
    
    \multirow{2}{*}{48-IHM} & F1  & 39.58 $\pm$ 1.12 & 32.97 $\pm$ 1.08 & 11.60 $\pm$ 0.54 & 43.76 $\pm$ 0.44 & 30.14 $\pm$ 1.15 & 40.21 $\pm$ 1.29 & 35.29 $\pm$ 0.30 & 43.36 $\pm$ 0.94 & \textbf{47.05 $\pm$ 0.48} & 45.42 $\pm$ 0.39 \\
                            & AUC & 76.60 $\pm$ 0.98 & 78.29 $\pm$ 0.25 & 67.08 $\pm$ 0.77 & 82.97 $\pm$ 0.61 & 71.34 $\pm$ 0.88 & 78.05 $\pm$ 0.72 & 80.45 $\pm$ 0.44 & 84.50 $\pm$ 0.33 & \textbf{85.55 $\pm$ 0.41} & 84.38 $\pm$ 0.34 \\
    \midrule
    \multirow{2}{*}{LOS}    & F1  & 58.52 $\pm$ 0.44 & 56.22 $\pm$ 0.28 & 42.01 $\pm$ 1.12 & 59.03 $\pm$ 0.39 & 57.59 $\pm$ 0.76 & 58.31 $\pm$ 0.46 & 56.96 $\pm$ 0.53 & 60.82 $\pm$ 0.56 & \textbf{61.61 $\pm$ 0.21} & 61.14 $\pm$ 0.33 \\
                            & AUC & 75.86 $\pm$ 0.66 & 72.22 $\pm$ 0.18 & 64.11 $\pm$ 0.40 & 76.17 $\pm$ 0.17 & 72.59 $\pm$ 1.07 & 72.59 $\pm$ 0.61 & 74.81 $\pm$ 0.37 & 76.80 $\pm$ 0.36 & 77.30 $\pm$ 0.78 & \textbf{77.98 $\pm$ 0.48} \\
    \midrule
    \multirow{2}{*}{25-PHE} & F1  & 27.77 $\pm$ 0.91 & 25.53 $\pm$ 0.11 & 26.52 $\pm$ 0.29 & 25.38 $\pm$ 0.52 & 12.45 $\pm$ 0.54 & 12.25 $\pm$ 0.55 & 24.61 $\pm$ 0.71 & 26.84 $\pm$ 0.56 & \textbf{28.52 $\pm$ 0.67} & 28.51 $\pm$ 0.47 \\
                            & AUC & 63.90 $\pm$ 1.31 & 62.57 $\pm$ 0.23 & 56.28 $\pm$ 0.16 & 72.50 $\pm$ 0.67 & 55.48 $\pm$ 0.39 & 58.52 $\pm$ 0.27 & 71.57 $\pm$ 0.47 & 74.73 $\pm$ 0.71 & \textbf{74.82 $\pm$ 0.77} & 74.49 $\pm$ 0.98 \\
    \bottomrule
  \end{tabular}
  }
\end{table*}

\begin{table}[h]
  \caption{Experiment setting II: Main results on CMU-MOSI}
  \label{tb:CMU-MOSI_Experiment II}
  \centering
  \resizebox{\textwidth}{!}{%
  \begin{tabular}{cc|ccccccc>{\columncolor{gray!20}}c>{\columncolor{gray!20}}c}
    \toprule
    Missing Rate   & Metric & ShaSpec & TF & LlMoE & FuseMoE-S & FuseMoE-L & FlexMoE & ConfSMoE & \textbf{TWIN-SF} & \textbf{TWIN-SP}  \\
    \midrule
    \multirow{2}{*}{0\%}  & F1  & 41.78 $\pm$ 0.64 & 43.28 $\pm$ 0.36 & 50.91 $\pm$ 0.19 & 47.69 $\pm$ 0.39 & 48.15 $\pm$ 0.57 & 51.50 $\pm$ 0.55 & 51.83 $\pm$ 0.27 & \textbf{53.35 $\pm$ 0.10} & 53.61 $\pm$ 0.39  \\
                          & AUC & 65.08 $\pm$ 0.39 & 64.48 $\pm$ 0.51 & 78.12 $\pm$ 0.31 & 75.26 $\pm$ 0.20 & 75.96 $\pm$ 0.29 & 79.38 $\pm$ 0.23 & 80.56 $\pm$ 0.10 & \textbf{81.23 $\pm$ 0.36} & 80.64 $\pm$ 0.10  \\
    \midrule
    \multirow{2}{*}{10\%} & F1  & 43.40 $\pm$ 0.74 & 40.14 $\pm$ 0.13 & 49.91 $\pm$ 0.33 & 46.81 $\pm$ 0.45 & 48.81 $\pm$ 0.42 & 49.92 $\pm$ 0.28 & 49.99 $\pm$ 0.85 & 51.66 $\pm$ 0.85 & \textbf{51.73 $\pm$ 1.40}  \\
                          & AUC & 65.45 $\pm$ 0.56 & 61.54 $\pm$ 0.66 & 77.12 $\pm$ 0.14 & 70.25 $\pm$ 0.93 & 72.36 $\pm$ 0.50 & 77.07 $\pm$ 0.21 & 78.06 $\pm$ 0.64 & \textbf{79.60 $\pm$ 0.40} & 78.40 $\pm$ 0.77  \\
    \midrule
    
    \multirow{2}{*}{20\%} & F1  & 39.15 $\pm$ 0.13 & 42.64 $\pm$ 0.30 & 48.84 $\pm$ 0.61 & 47.92 $\pm$ 0.58 & 45.00 $\pm$ 0.44 & 49.16 $\pm$ 0.24 & 50.58 $\pm$ 0.73 & \textbf{51.17 $\pm$ 0.42} & 50.66 $\pm$ 0.84  \\
                          & AUC & 60.87 $\pm$ 0.29 & 65.52 $\pm$ 0.98 & 76.64 $\pm$ 0.23 & 71.05 $\pm$ 0.36 & 72.15 $\pm$ 0.37 & 73.80 $\pm$ 1.43 & 77.17 $\pm$ 0.55 & \textbf{78.26 $\pm$ 0.49} & 76.33 $\pm$ 0.82  \\
    \midrule

    \multirow{2}{*}{30\%} & F1  & 39.54 $\pm$ 0.17 & 40.42 $\pm$ 0.75 & 46.13 $\pm$ 0.55 & 43.78 $\pm$ 0.51 & 44.48 $\pm$ 0.63 & 45.49 $\pm$ 0.44 & 47.57 $\pm$ 0.10 & \textbf{49.89 $\pm$ 0.79} & 49.61 $\pm$ 1.21  \\
                          & AUC & 60.80 $\pm$ 0.44 & 60.57 $\pm$ 0.99 & 74.33 $\pm$ 0.49 & 70.38 $\pm$ 0.56 & 72.25 $\pm$ 0.48 & 72.54 $\pm$ 0.92 & 75.59 $\pm$ 0.18 & \textbf{76.90 $\pm$ 0.72} & 74.45 $\pm$ 0.95  \\
    \midrule
    
    \multirow{2}{*}{40\%} & F1  & 35.32 $\pm$ 0.32 & 38.69 $\pm$ 0.42 & 43.31 $\pm$ 0.38 & 41.88 $\pm$ 0.42 & 43.20 $\pm$ 0.20 & 43.13 $\pm$ 0.61 & 45.12 $\pm$ 0.27 & 46.69 $\pm$ 2.12 & \textbf{47.52 $\pm$ 1.88}  \\
                          & AUC & 59.58 $\pm$ 0.13 & 61.62 $\pm$ 0.07 & 67.21 $\pm$ 0.43 & 64.61 $\pm$ 0.63 & 66.14 $\pm$ 0.11 & 68.12 $\pm$ 0.34 & 71.65 $\pm$ 0.47 & \textbf{72.53 $\pm$ 2.65} & 72.22 $\pm$ 1.21  \\
    \midrule
    
    \multirow{2}{*}{50\%} & F1  & 36.39 $\pm$ 0.97 & 38.46 $\pm$ 1.09 & 41.54 $\pm$ 0.77 & 38.55 $\pm$ 0.34 & 40.90 $\pm$ 0.49 & 41.95 $\pm$ 0.48 & 44.34 $\pm$ 0.77 & 44.29 $\pm$ 1.92 & \textbf{46.62 $\pm$ 4.43}  \\
                          & AUC & 55.60 $\pm$ 0.12 & 60.49 $\pm$ 0.80 & 67.43 $\pm$ 0.34 & 62.28 $\pm$ 0.28 & 65.73 $\pm$ 0.30 & 66.47 $\pm$ 0.49 & 70.41 $\pm$ 1.31 & \textbf{71.03 $\pm$ 2.43} & 70.50 $\pm$ 1.30  \\
    \bottomrule
  \end{tabular}
  }
\end{table}

\begin{table}[h]
  \caption{Experiment setting III: Main results on CMU-MOSI}
  \label{tb:CMU-MOSI_Experiment III AUC}
  \centering
  \resizebox{\textwidth}{!}{%
  \begin{tabular}{ccc|ccccccc>{\columncolor{gray!20}}c>{\columncolor{gray!20}}c}
    \toprule
     \multicolumn{3}{c|}{Test Modality} & \multicolumn{9}{c}{Dataset: CMU-MOSI / Evaluation Metric: AUC }  \\
    \midrule
    Video     & Text        & Audio       & ShaSpec  & TF & LlMoE & FuseMoE-S & FuseMoE-L & FlexMoE & ConfSMoE & \textbf{TWIN-SF} & \textbf{TWIN-SP}  \\
    \midrule
    \checkmark & \checkmark &             & 62.18 $\pm$ 0.88 & 59.37 $\pm$ 1.12 & 74.63 $\pm$ 0.13 & 55.35 $\pm$ 1.11 & 56.44 $\pm$ 1.48 & 74.59 $\pm$ 1.50 & 75.78 $\pm$ 1.03 & \textbf{78.75 $\pm$ 0.25} & 77.10 $\pm$ 0.57 \\
    \checkmark &            & \checkmark  & 55.79 $\pm$ 2.53 & 57.87 $\pm$ 0.32 & 56.84 $\pm$ 0.34 & 56.24 $\pm$ 1.96 & 55.11 $\pm$ 0.77 & 52.28 $\pm$ 1.23 & 57.61 $\pm$ 0.62 & \textbf{69.90 $\pm$ 1.03} & 59.06 $\pm$ 0.89 \\
               & \checkmark & \checkmark  & 65.22 $\pm$ 1.01 & 57.89 $\pm$ 0.65 & 75.59 $\pm$ 0.39 & 54.36 $\pm$ 1.67 & 56.10 $\pm$ 1.76 & 75.55 $\pm$ 1.64 & 76.53 $\pm$ 0.74 & \textbf{77.87 $\pm$ 0.46} & 77.69 $\pm$ 0.22 \\
    \midrule
    \checkmark &            &             & 55.10 $\pm$ 1.59 & 58.57 $\pm$ 0.44 & 74.73 $\pm$ 0.51 & 55.94 $\pm$ 1.51 & 55.66 $\pm$ 1.66 & 74.35 $\pm$ 1.60 & 75.06 $\pm$ 0.22 & \textbf{77.76 $\pm$ 0.46} & 77.06 $\pm$ 0.65 \\
               & \checkmark &             & 61.60 $\pm$ 4.80 & 60.28 $\pm$ 1.42 & 75.73 $\pm$ 0.33 & 55.45 $\pm$ 0.95 & 56.58 $\pm$ 1.29 & 74.42 $\pm$ 1.39 & 74.88 $\pm$ 0.46 & \textbf{75.59 $\pm$ 0.81} & 77.05 $\pm$ 0.72 \\
               &            & \checkmark  & 66.85 $\pm$ 3.01 & 57.19 $\pm$ 0.72 & 73.29 $\pm$ 0.53 & 56.04 $\pm$ 1.91 & 57.85 $\pm$ 0.74 & 75.34 $\pm$ 1.46 & 76.47 $\pm$ 0.63 & \textbf{77.98 $\pm$ 1.15} & 78.03 $\pm$ 0.11 \\
    \bottomrule
  \end{tabular}
  }
\end{table}

\begin{table}
  \caption{Experiment setting III: Main results on CMU-MOSI}
  \label{tb:CMU-MOSI_Experiment III F1}
  \centering
  \resizebox{\textwidth}{!}{%
  \begin{tabular}{ccc|ccccccc>{\columncolor{gray!20}}c>{\columncolor{gray!20}}c}
    \toprule
     \multicolumn{3}{c|}{Test Modality} & \multicolumn{9}{c}{Dataset: CMU-MOSI / Evaluation Metric: F1}  \\
    \midrule
    Video     & Text        & Audio       & ShaSpec & TF & LlMoE & FuseMoE-S & FuseMoE-L & FlexMoE & \textbf{ConfSMoE} & \textbf{TWIN-SF} & \textbf{TWIN-SP}   \\
    \midrule
    \checkmark & \checkmark &             & 41.27 $\pm$ 0.81 & 34.96 $\pm$ 0.72 & 44.45 $\pm$ 0.33 & 30.51 $\pm$ 2.53 & 29.90 $\pm$ 2.58 & 50.54 $\pm$ 0.73 & 50.26 $\pm$ 0.28 & \textbf{51.35 $\pm$ 0.51} & 53.17 $\pm$ 0.64 \\
    \checkmark &            & \checkmark  & 36.99 $\pm$ 2.03 & 34.83 $\pm$ 0.54 & 35.20 $\pm$ 0.64 & 24.29 $\pm$ 4.23 & 27.89 $\pm$ 6.23 & 30.02 $\pm$ 7.82 & 37.37 $\pm$ 1.00 & \textbf{45.10 $\pm$ 2.82} & 37.99 $\pm$ 0.33 \\
               & \checkmark & \checkmark  & 43.72 $\pm$ 1.64 & 37.25 $\pm$ 0.53 & 47.79 $\pm$ 0.46 & 20.68 $\pm$ 2.14 & 29.88 $\pm$ 4.55 & 50.41 $\pm$ 1.55 & 50.77 $\pm$ 0.48 & \textbf{51.11 $\pm$ 0.28} & 53.86 $\pm$ 3.42 \\
    \midrule
    \checkmark &            &             & 42.98 $\pm$ 0.21 & 38.11 $\pm$ 0.35 & 44.86 $\pm$ 0.94 & 25.67 $\pm$ 2.82 & 27.34 $\pm$ 2.94 & 50.39 $\pm$ 0.87 & 50.41 $\pm$ 0.24 & \textbf{51.55 $\pm$ 0.32} & 53.10 $\pm$ 0.65 \\
               & \checkmark &             & 41.36 $\pm$ 2.27 & 37.99 $\pm$ 0.40 & 43.81 $\pm$ 0.42 & 25.70 $\pm$ 2.82 & 27.30 $\pm$ 2.95 & 50.65 $\pm$ 0.90 & 50.35 $\pm$ 0.18 & \textbf{50.72 $\pm$ 0.34} & 52.78 $\pm$ 1.92  \\
               &            & \checkmark  & 45.45 $\pm$ 2.73 & 38.79 $\pm$ 0.56 & 45.50 $\pm$ 0.47 & 25.23 $\pm$ 4.59 & 29.73 $\pm$ 3.51 & 50.57 $\pm$ 1.35 & 50.69 $\pm$ 0.47 & \textbf{52.40 $\pm$ 1.76} & 53.66 $\pm$ 3.14 \\
    \bottomrule
  \end{tabular}
  }
\end{table}

\begin{table}  
  \caption{Experiment setting III: Main results on CMU-MOSEI}
  \label{tb:CMU-MOSEI_Experiment III AUC}
  \centering
  \resizebox{\textwidth}{!}{%
  \begin{tabular}{ccc|ccccccc>{\columncolor{gray!20}}c>{\columncolor{gray!20}}c}
    \toprule
     \multicolumn{3}{c|}{Test Modality} & \multicolumn{9}{c}{Dataset: CMU-MOSEI / Evaluation Metric: AUC }  \\
    \midrule
    Video     & Text        & Audio       & ShaSpec & TF & LlMoE & FuseMoE-S & FuseMoE-L & FlexMoE & ConfSMoE & TWIN-SF & \textbf{TWIN-SP}  \\
    \midrule
    \checkmark & \checkmark &             & 64.47 $\pm$ 0.10 & 56.98 $\pm$ 0.20 & 81.23 $\pm$ 0.07 & 75.81 $\pm$ 0.82 & 76.62 $\pm$ 0.37 & 80.63 $\pm$ 0.41 & 81.52 $\pm$ 0.21 & \textbf{81.67 $\pm$ 0.61} & 81.50 $\pm$ 0.17 \\
    \checkmark &            & \checkmark  & 58.41 $\pm$ 0.10 & 58.41 $\pm$ 0.17 & 59.64 $\pm$ 0.56 & 57.32 $\pm$ 0.62 & 57.22 $\pm$ 0.12 & 58.28 $\pm$ 1.51 & 60.07 $\pm$ 0.22 & 60.53 $\pm$ 1.15 & \textbf{60.64 $\pm$ 0.17} \\
               & \checkmark & \checkmark  & 64.46 $\pm$ 0.14 & 58.64 $\pm$ 0.13 & 80.00 $\pm$ 0.29 & 77.24 $\pm$ 0.60 & 77.31 $\pm$ 0.16 & 80.00 $\pm$ 0.10 & \textbf{81.26 $\pm$ 0.10} & 81.24 $\pm$ 0.13 & 81.23 $\pm$ 0.16 \\
    \midrule
    \checkmark &            &             & 63.52 $\pm$ 0.19 & 58.71 $\pm$ 0.35 & 80.69 $\pm$ 0.34 & 76.84 $\pm$ 0.14 & 76.27 $\pm$ 0.70 & 81.06 $\pm$ 0.20 & 81.53 $\pm$ 0.02 & \textbf{81.62 $\pm$ 0.37} & 81.51 $\pm$ 0.16 \\
               & \checkmark &             & 64.40 $\pm$ 0.24 & 58.89 $\pm$ 0.40 & 80.96 $\pm$ 0.26 & 76.85 $\pm$ 0.14 & 76.51 $\pm$ 0.50 & 80.48 $\pm$ 0.21 & 81.50 $\pm$ 0.16 & \textbf{81.67 $\pm$ 0.45} & 81.61 $\pm$ 0.17 \\
               &            & \checkmark  & 65.16 $\pm$ 0.04 & 58.32 $\pm$ 0.04 & 80.19 $\pm$ 0.11 & 77.13 $\pm$ 0.75 & 77.23 $\pm$ 0.24 & 80.21 $\pm$ 0.11 & 80.87 $\pm$ 0.39 & \textbf{81.29 $\pm$ 0.14} & 81.23 $\pm$ 0.17 \\
    \bottomrule
  \end{tabular}
  }
\end{table}

\begin{table} 
  \caption{Experiment setting III: Main results on CMU-MOSEI}
  \label{tb:CMU-MOSEI_Experiment III F1}
  \centering
  \resizebox{\textwidth}{!}{%
  \begin{tabular}{ccc|ccccccc>{\columncolor{gray!20}}c>{\columncolor{gray!20}}c}
    \toprule
     \multicolumn{3}{c|}{Test Modality} & \multicolumn{9}{c}{Dataset: CMU-MOSEI / Evaluation Metric: F1}  \\
    \midrule
    Video     & Text        & Audio       & ShaSpec & TF & LlMoE & FuseMoE-S & FuseMoE-L & FlexMoE & ConfSMoE & TWIN-SF & \textbf{TWIN-SP}  \\
    \midrule
    \checkmark & \checkmark &             & 40.75 $\pm$ 0.14 & 31.38 $\pm$ 0.61 & 58.97 $\pm$ 0.53 & 55.85 $\pm$ 0.69 & 53.59 $\pm$ 1.07 & 60.10 $\pm$ 0.76 & 61.01 $\pm$ 0.58 & \textbf{61.97 $\pm$ 0.54} & 61.06 $\pm$ 1.00 \\
    \checkmark &            & \checkmark  & 38.15 $\pm$ 0.92 & 31.73 $\pm$ 0.16 & 38.26 $\pm$ 2.40 & 29.40 $\pm$ 1.67 & 33.29 $\pm$ 2.89 & 36.28 $\pm$ 3.51 & 39.20 $\pm$ 0.93 & \textbf{40.64 $\pm$ 0.14} & 40.18 $\pm$ 0.11 \\
               & \checkmark & \checkmark  & 40.82 $\pm$ 0.12 & 31.86 $\pm$ 0.23 & 58.58 $\pm$ 0.31 & 57.02 $\pm$ 1.60 & 53.38 $\pm$ 0.65 & 58.70 $\pm$ 2.10 & 61.37 $\pm$ 0.28 & \textbf{61.98 $\pm$ 0.66} & 61.54 $\pm$ 0.57 \\
    \midrule
    \checkmark &            &             & 40.37 $\pm$ 0.46 & 31.59 $\pm$ 0.23 & 61.20 $\pm$ 0.46 & 54.02 $\pm$ 2.88 & 53.08 $\pm$ 0.40 & 58.58 $\pm$ 0.71 & \textbf{61.71 $\pm$ 1.00} & 61.05 $\pm$ 1.40 & 61.05 $\pm$ 1.00 \\
               & \checkmark &             & 42.16 $\pm$ 0.22 & 30.84 $\pm$ 0.26 & 59.35 $\pm$ 0.12 & 55.87 $\pm$ 0.68 & 54.11 $\pm$ 0.49 & 58.42 $\pm$ 0.27 & 61.34 $\pm$ 0.83 & \textbf{61.97 $\pm$ 0.54} & 61.00 $\pm$ 0.97 \\
               &            & \checkmark  & 42.45 $\pm$ 0.85 & 30.71 $\pm$ 0.71 & 60.00 $\pm$ 0.85 & 56.64 $\pm$ 1.65 & 52.92 $\pm$ 1.11 & 59.50 $\pm$ 0.80 & 60.91 $\pm$ 0.59 & \textbf{61.36 $\pm$ 0.47} & 61.40 $\pm$ 0.63 \\
    \bottomrule
  \end{tabular}
  }
\end{table}


\begin{table}[t]
\centering
\caption{Comparison of pre-imputation coefficient operator $\mathcal{B}_m$ on CMU-MOSI
(Experiment Setting II, missing ratio $0.5$). The coefficient
$\Delta_m$ is plugged into $\hat{h}_m^{pre}=\mu_m+\sigma_m \odot \Delta_m$.}
\label{tab:preimpute}
\begin{tabular}{llcc}
\toprule

$\mathcal{B}_m$ & AUC & F1 \\
\midrule
arcsinh     & $69.61 \pm 2.12$ & $45.95 \pm 0.97$ \\
tanh        & $69.52 \pm 2.02$ & $45.76 \pm 1.18$ \\
softsign    & $69.63 \pm 1.72$ & $46.19 \pm 1.61$  \\
gated       & $70.28 \pm 1.47$ & $44.24 \pm 1.10$  \\
Ours        & $\mathbf{70.50 \pm 1.30}$ & $\mathbf{46.62 \pm 4.43}$ \\
\bottomrule
\end{tabular}
\end{table}


\begin{figure}
    \centering
    \includegraphics[width=\linewidth]{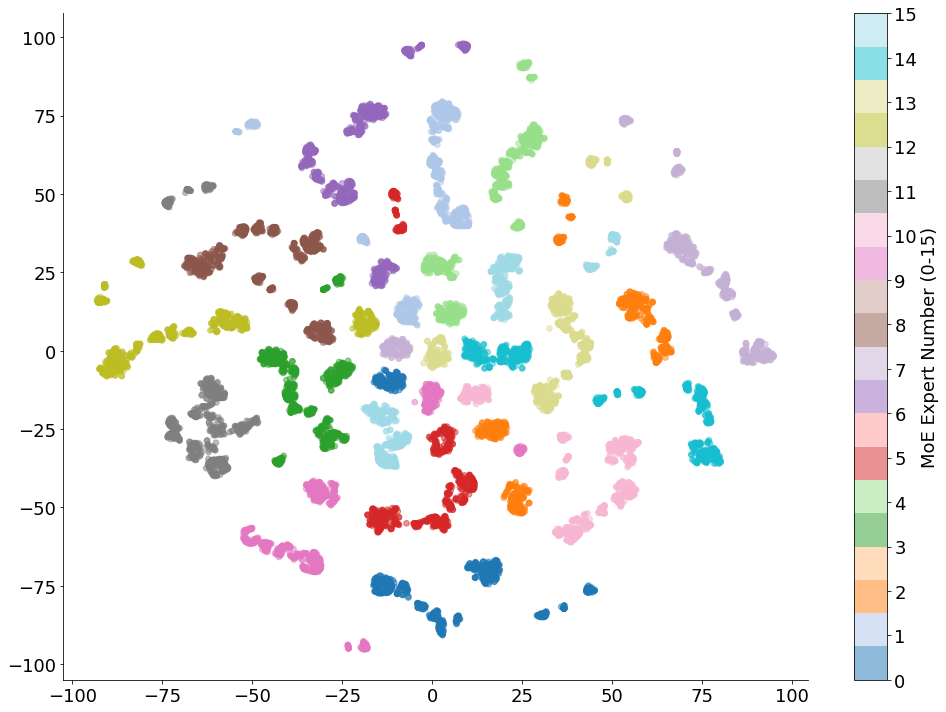}
    \caption{Embedding space of MoE has been split into different clusters}
    \label{fig:t-sne_plot_sub_spaces}
\end{figure}

\end{document}